\documentclass{article}

\usepackage{wrapfig}

    \usepackage[final]{neurips_2022}

\usepackage[utf8]{inputenc} 
\usepackage[T1]{fontenc}    
\usepackage{hyperref}       
\usepackage{url}            
\usepackage{booktabs}       
\usepackage{amsfonts}       
\usepackage{nicefrac}       
\usepackage{microtype}      
\usepackage{xcolor}         

\usepackage{microtype}
\usepackage{graphicx}
\usepackage{booktabs} 

\usepackage{hyperref}

\usepackage[utf8]{inputenc} 
\usepackage[T1]{fontenc}    
\usepackage{hyperref}       
\usepackage{url}            
\usepackage{booktabs}       
\usepackage{amsfonts}       
\usepackage{nicefrac}       
\usepackage{microtype}      
\usepackage{xcolor}         

\usepackage{microtype}
\usepackage{graphicx}
\usepackage{booktabs} 
\usepackage{algorithm}
\usepackage{algorithmic}

\usepackage{caption}
\usepackage{subcaption}

\usepackage{amsmath}
\usepackage{amsfonts}
\usepackage{mathtools}
\usepackage{amsthm}
\usepackage{thmtools}
\usepackage{delimset}
\usepackage{xparse}
\usepackage{xcolor}

\usepackage{tabularx}

\newcommand{\A}{\mathcal{A}}
\newcommand{\X}{\mathcal{X}}

\newcommand{\D}{\mathcal{D}}

\newcommand{\Z}{\mathcal{Z}}
\newcommand{\R}{\mathbb{R}}

\newcommand\diag[1]{\text{diag}\brk*{#1}}
\newcommand\trace[1]{\text{trace}\brk*{#1}}
\renewcommand\abs[1]{\left|#1\right|}
\renewcommand\norm[1]{\left\lVert#1\right\rVert}

\DeclarePairedDelimiterX{\infdivx}[2]{(}{)}{%
  #1\;\delimsize|\delimsize|\;#2%
}
\newcommand{\kld}[2]{\ensuremath{D_{KL}\infdivx{#1}{#2}}}

\DeclareDocumentCommand \E { o m } {%
  \ensuremath{\mathbb{E}%
  \IfValueTF {#1} {%
    _{#1} \brk[s]2{ #2 }%
  }{%
    \left[ #2 \right]%
  }%
  }%
}

\newtheorem*{theorem*}{Theorem}
\newtheorem{lemma}{Lemma}
\newtheorem{corollary}{Corollary}
\newtheorem{definition}{Definition}

\newtheorem*{remark*}{Remark}

\newtheorem{proposition}{Proposition}

\usepackage{hyperref}
\usepackage{cleveref}

\title{Uncertainty Estimation Using Riemannian Model~Dynamics for Offline Reinforcement Learning}

\author{%
  Guy Tennenholtz\thanks{Correspondence to \texttt{guytenn@gmail.com}} \\
  Technion Institute of Technology \\
  \And
  Shie Mannor \\
  Technion Institute of Technology \& Nvidia Research
}

\begin{document}

\maketitle



\begin{abstract}
    Model-based offline reinforcement learning approaches generally rely on bounds of model error. Estimating these bounds is usually achieved through uncertainty estimation methods. In this work, we combine parametric and nonparametric methods for uncertainty estimation through a novel latent space based metric. In particular, we build upon recent advances in Riemannian geometry of generative models to construct a pullback metric of an encoder-decoder based forward model. Our proposed metric measures both the quality of out-of-distribution samples as well as the discrepancy of examples in the data. We leverage our method for uncertainty estimation in a pessimistic model-based framework, showing a significant improvement upon contemporary model-based offline approaches on continuous control and autonomous driving benchmarks. 
\end{abstract}

\section{Introduction}
\label{section: introduction}

Offline Reinforcement Learning (RL) \citep{levine2020offline}, a.k.a. batch-mode RL \citep{ernst2005tree,riedmiller2005neural,fonteneau2013batch}, involves learning a policy from data sampled by a potentially suboptimal policy. Offline RL seeks to {\em surpass} the average performance of the agents that generated the data. Traditional methodologies fall short in offline settings, causing overestimation of the return \citep{buckman2020importance,wang2020statistical,zanette2020exponential}. 

One approach to overcome this in model-based settings is to penalize the return in out of distribution (OOD) regions, as depicted in \Cref{fig: data manifold}. In this manner, the agent is constrained to stay ``near" areas of low model error, thereby limiting possible overestimation. However, reliable estimates of model error are key to the success of such methods.

Estimating model error in OOD regions can be achieved through uncertainty estimation \citep{yu2020mopo}. Methods of parametric uncertainty estimation such as bootstrap ensembles \citep{efron1982jackknife}, Monte Carlo Dropout \citep{gal2016dropout}, and randomized priors \citep{osband2018randomized}, may be susceptible to poor model specification and are most effective when dealing with large datasets. In contrast, nonparametric methods such as k-nearest neighbors (k-NN) \citep{villa2013reliability,fathabadi2021comparison} are beneficial in regions of limited data, yet require a proper metric to be used. 

We propose to combine parametric and nonparametric methods for uncertainty estimation. Particularly, we define a novel Riemmannian metric which captures the epistemic and aleatoric uncertainty of a generative parametric forward model. This distance metric is then applied to measure the average geodesic distance to the $k$-nearest neighbors in the data. We derive analytical expressions for our metric and provide an efficient manner to estimate it. We then demonstrate the effectiveness of our metric for penalizing an offline RL agent compared to contemporary approaches on continuous control and autonomous driving benchmarks. As we empirically show, common approaches, including statistical bootstrap ensemble or Euclidean distances in latent space, do not necessarily capture the underlying degree of error needed for model-based offline RL.

\section{Preliminaries}
\label{section: preliminaries}

\subsection{Offline Reinforcement Learning}
\label{section: offline rl}

We consider the standard Markov Decision Process (MDP) framework \citep{puterman2014markov} defined by the tuple $(\mathcal{S}, \A, r, P, \alpha)$, where $\mathcal{S}$ is the state space, $\A$ the action space, ${r:\mathcal{S}\times\A \mapsto [0, 1]}$ the reward function, ${P:\mathcal{S}\times\A\times\mathcal{S} \mapsto [0, 1]}$ the transition kernel, and ${\alpha \in (0, 1)}$ is the discount factor. 

In the online setting of reinforcement learning (RL), the environment initiates at some state $s_0 \sim \rho_0$. At any time step the environment is in a state $s \in \mathcal{S}$, an agent takes an action $a \in \A$ and receives a reward $r(s, a)$ from the environment as a result of this action. The environment transitions to state $s'$ according to the transition function $P(\cdot | s, a)$. The goal of online RL is to find a policy $\pi(a|s)$ that maximizes the expected discounted return
${
    v^\pi = \E[\pi]{\sum_{t=0}^\infty \alpha^t r(s_t, a_t) | s_0 \sim \rho_0}.
}
$

Unlike the online setting, the offline setup considers a dataset ${\D_n = \brk[c]*{s_i, a_i, r_i, s'_i}_{i=1}^n}$ of transitions generated by some unknown agents. The objective of offline RL is to find the best policy in the test environment (i.e., real MDP) given only access to the data generated by the unknown agents.


\subsection{Riemannian Manifolds}
\label{section: riemannian manifolds}

We define the Riemannian pullback metric, a fundamental component of our proposed method. We refer the reader to \citet{carmo1992riemannian} for further details on Riemannian geometry. 

We are interested in studying a smooth surface $M$ with a Riemannian metric $g$. A Riemannian metric is a smooth function that assigns a symmetric positive definite matrix to any point in $M$. At each point $z \in M$ a tangent space $T_zM$ specifies the pointing direction of vectors “along” the surface. 


\begin{definition}
    Let $M$ be a smooth manifold. A Riemannian metric $g$ on $M$ changes smoothly and defines a real scalar product on the tangent space $T_z M$ for any $z \in M$ as
    \begin{align*}
        g_z(x,y) = \brk[a]*{x, y}_{z} = \brk[a]*{x, G(z)y}, \quad x,y \in T_z M,
    \end{align*}
    where $G(z) \in \R^{d_z\times d_z}$ is the corresponding metric tensor. $(M, g)$ is called a Riemannian manifold.
\end{definition}

\begin{wrapfigure}[16]{R}{0pt}
\centering
\includegraphics[width=0.45\textwidth]{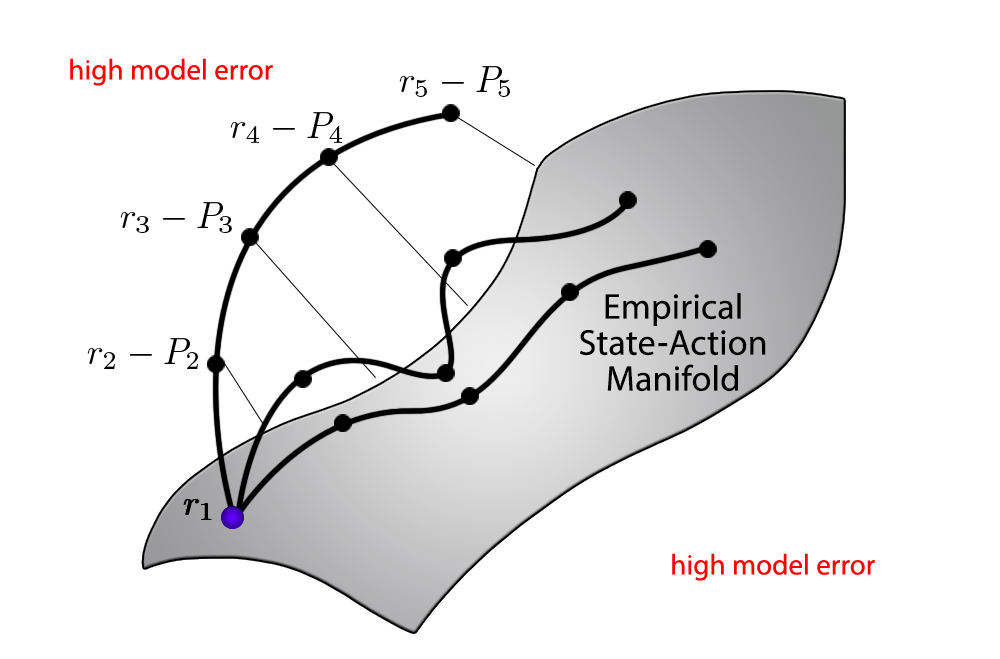}
\caption{\label{fig: data manifold} \footnotesize Illustration of a data manifold corresponding to model error. Curved lines represent trajectories at different stages of training. Agent incurs a penalty in areas of high model error. }
\end{wrapfigure}

The Riemannian metric enables us to easily define geodesic curves. Consider some differentiable mapping ${\gamma: [0, 1] \mapsto M \subseteq \R^{d_z}}$, such that $\gamma(0) = z_0, \gamma(1) = z_1$. The length of the curve $\gamma$ measured on $M$ is given by 
\begin{align}
    L(\gamma) = \int_0^1 \sqrt{\brk[a]*{\frac{\partial \gamma(t)}{\partial t}, G({\gamma(t)}) \frac{\partial \gamma(t)}{\partial t}}} dt.
    \label{eq: curve length}
\end{align}
The geodesic distance $d(z_1, z_2)$ between any two points $z_1,z_2 \in M$ is then the infimum length over all
curves $\gamma$ for which $\gamma(0) = z_0, \gamma(1) = z_1$. That is,
$
    d(z_1, z_2) = \inf_{\gamma} L(\gamma) \quad \text{s.t. } \gamma(0) = z_0, \gamma(1) = z_1.
$
The geodesic distance can be found by solving a system of nonlinear ordinary differential equations (ODEs) defined in the intrinsic coordinates \citep{carmo1992riemannian}.

\textbf{Pullback Metric.} Assume an ambient (observation) space $\mathcal{X}$ and its respective Riemannian manifold $(M_\X, g_\X)$. Learning $g_\X$ can be hard (e.g., learning the distance metric between images). Still, it may be captured through a low dimensional submanifold. As such, it is many times convenient to parameterize the surface $M_\X$ by a latent space $\Z = \R^{d_\Z}$ and a smooth function $f: \Z \mapsto \X$, where $\Z$ is a low dimensional latent embedding space. As learning the manifold $M_\X$ can be hard, we turn to learning the immersed low dimensional submanifold $M_\Z$ (for which the chart maps are trivial, since $\Z = \R^{d_\Z}$). Given a curve $\gamma:[0,1] \mapsto M_\Z$ we have that 
$
    \brk[a]*{\frac{\partial f(\gamma(t))}{\partial t}, G_\X({f(\gamma(t))}) \frac{\partial f(\gamma(t))}{\partial t}} 
    = 
    \brk[a]*{\frac{\partial \gamma(t)}{\partial t}, J_f^T(\gamma(t)) G_\X({f(\gamma(t))}) J_f(\gamma(t)) \frac{\partial \gamma(t)}{\partial t}},
$
where the Jacobian matrix $J_f(z) = \frac{\partial f}{\partial z} \in \R^{d_\X \times d_\Z}$ maps tangent vectors in $T{M_\Z}$ to tangent vectors in $T{M_\X}$. The induced metric is thus given by 
\begin{align}
    \label{eq: pullback metric}
    {G_f(z) = J_f(z)^T G_\X({f(z)}) J_f(z)}.
\end{align}
The metric $G_f$ is known as the \textit{pullback metric}, as it ``pulls back" the metric $G_\X$ on $\X$ back to $G_f$ via $f: \Z \mapsto \X$. The pullback metric captures the intrinsic geometry of the immersed submanifold while taking into account the ambient space~$\X$. The geodesic distance in ambient space is captured by geodesics in the latent space $\Z$, reducing the problem to learning the latent embedding space $\Z$ and the observation function $f$. Indeed, learning the latent space and observation function $f$ can be achieved through a encoder-decoder framework, such as a VAE \citep{arvanitidis2018latent}.


\section{Background: Penalty of Uncertainty for Offline Reinforcement Learning}
\label{section: metric}

A key element of model-based RL methods involves estimating a model $\hat{P}(s'|s,a)$ to construct a pessimistic MDP\footnote{Pessimism is a key element of offline RL algorithms \citep{jin2020pessimism}, limiting overestimation of a trained policy due to the distribution shift between the data and the trained policy.}.  This work builds upon MOPO, a recently proposed model-based offline RL framework \citep{yu2020mopo}). Particularly, we assume access to an approximate MDP $(\mathcal{S}, \A, \hat{r}, \hat{P}, \alpha)$ (e.g., trained by maximizing the likelihood of the data), and define a penalized MDP $(\mathcal{S}, \A, \tilde{r}, \hat{P}, \alpha)$, such that for all $s \in \mathcal{S}, a \in \A$,
$
    \tilde{r}(s,a) = \hat{r}(s,a) - \lambda c(P(\cdot | s, a), \hat{P}(\cdot | s, a)),
$
where $c$ penalizes the reward according to model error (e.g., the total variation distance) and $\lambda > 0$. The offline RL problem is then solved by executing an online algorithm in the reward-penalized MDP. Unfortunately, as $P(\cdot | s, a)$ is unknown, and can only be estimated from the data, $c(P(\cdot | s, a), \hat{P}(\cdot | s, a))$ cannot be calculated. Nevertheless, one can attempt to upper bound the distance, i.e., for some $U: \mathcal{S} \times \A \mapsto \R$,
$
    c(P(\cdot | s, a), \hat{P}(\cdot | s, a)) \leq U(s,a), \forall s\in \mathcal{S}, a \in \A.
$ In this work we propose to use a naturally induced metric of a variational forward model, which we show can introduce an effective penalty for offline RL. 
In \Cref{section: metric} we define this metric, and finally, we leverage it in \Cref{section: GELATO}. 

\section{Metrics of Uncertainty}
\label{section: metric}

As described in the previous section, our goal is to estimate model error in order to penalize the agent in out of distribution (OOD) regions. \citet{yu2020mopo} proposed to achieve this through bootstrap ensembles, an out of distribution uncertainty estimation technique. Alternatively, we propose to employ a well-known nonparametric approach for uncertainty estimation \citep{villa2013reliability,fathabadi2021comparison}, namely $k$-nearest neighbors ($k$-NN). Specifically, for any $s, a \in \mathcal{S} \times \A$, we estimate model error by
\begin{align}
\label{eq: nearest neighbors}
    U(s,a) = \frac{1}{k} \sum_{(s_i,a_i) \in \text{NN}_k(s,a)} d((s,a),(s_i, a_i)),
\end{align}
where ${d: \mathcal{S} \times \A \times \mathcal{S} \times \A \mapsto \R_+}$ is a distance metric, and $\text{NN}_k(s,a)$ is the set of $k$-nearest neighbors of $(s,a)$ in $\D_n$ according to the distance metric $d$.

A question arises: how to choose $d$? Using the Euclidean distance in ambient (state-action) space is usually a bad choice (e.g., the $\ell_2$ distance between natural images is not necessarily meaningful). Moreover, to correctly measure the error, model dynamics should be somehow taken into consideration. We therefore consider an alternative approach which leverages the latent space of a variational forward model, as described next.

\subsection{A Variational Latent Model of Dynamics}

We begin by modeling $\hat{P}(s'|s,a)$ using a generative latent model. Specifically, we consider a latent model which consists of an encoder $E : \mathcal{S} \times \mathcal{A} \mapsto \mathcal{B}(\Z)$ and a decoder $f_D : \mathcal{Z} \mapsto \mathcal{B}(\mathcal{S})$, where $\mathcal{B}(\X)$ is set of probability measures on the Borel sets of $\X$. While the encoder $E$ learns a latent representation of $s, a$, the decoder $f_D$ estimates the next state $s'$ according to $P(\cdot | s, a)$. This model corresponds to the decomposition $P(\cdot | s, a) = f_D(\cdot | E(s,a))$. Such a model can be trained by maximizing the Evidence Lower BOund (ELBO, \citet{kingma2013auto}) over the data. That is, given a prior $P(z)$, we model $E_\phi, f_{D,\theta}$ as parametric functions and maximize the ELBO,
$
    \max_{\theta, \phi} \E[E_\phi(z|s,a)]{\log f_{D,\theta} (s'|z) } - D_{KL}(E_\phi(z|s,a) || P(z))
$
We refer the reader to the appendix for an exhaustive overview of training VAEs by maximum likelihood and the ELBO.

Recall that we wish to find a good metric for estimating model error. Having learned a latent model for $\hat{P}(s'|s,a)$, its latent space $\Z$ can be used to define the metric $d$ in \Cref{eq: nearest neighbors}, i.e., the Euclidean distance between latent representations of state-action pairs. Unfortunately, as was previously shown \citep{arvanitidis2018latent}, latent codes in variational models contain sharp discontinuities, rendering Euclidean distances in latent space unreliable and inaccurate (as we will also demonstrate in our experiments). Instead, we propose to use the natural induced metric of our latent model, as described in the following subsection.


\begin{algorithm}[t!]
   \caption{GELATO: Geometrically Enriched LATent model for Offline reinforcement learning}
   \label{alg: gelato}
\begin{algorithmic}[1]
\small
   \STATE {\bfseries Input:} Offline dataset $\D_n$, RL algorithm
   \STATE Train variational latent forward model on dataset $\D_n$ by maximizing ELBO.
   \STATE Construct approximate MDP $(\mathcal{S}, \A, \hat{r}, \hat{P}, \alpha)$
   \STATE Use distance $d_\Z$ induced by pullback metric $G_{f_{D,U} \circ f_F}$ (\Cref{prop: metric comp}) to penalize reward ${\tilde{r}_d(s, a) = \hat{r}(s,a) - \lambda U(s,a)}$ where ${U(s,a) = \sum_{(s_i,a_i) \in \text{NN}_k(s,a)} d_\Z\brk*{E(s,a), E(s_i, a_i)}}$
   \STATE Train RL algorithm over penalized MDP $(\mathcal{S}, \A, \tilde{r}_d, \hat{P}, \alpha)$
\end{algorithmic}
\end{algorithm}

\subsection{The Pullback Metric of Model Dynamics}

In this part we define the metric $d$ that will be used to approximate model error in \Cref{eq: nearest neighbors}. Specifically, we consider a Riemannian submanifold defined by a latent space $\Z$ and observation function $f$, which induces minimum energy in the ambient space. We will later choose $\Z$ to be the latent space of our variational model (i.e., encoded state-action) and $f$ to be the decoder function $f_D$ of next state transitions. We define the distance metric formally below.


\begin{definition}
\label{definition: distance}
We define a Riemannian submanifold $(\mathcal{M}_\Z, g_\Z)$ by a differential function $f:\Z \mapsto \mathcal{S}$ and latent space $\Z$ such that 
\begin{align*}
    d_\Z(z_1, z_2) = \inf_\gamma \int_0^1 \norm{\frac{\partial f(\gamma(t))}{\partial t}} dt
    \quad
    \text{s.t. } \gamma(0) = z_1, \gamma(1) = z_2.
\end{align*}
\end{definition}

A similar metric has been used in previous work on generative latent models \citep{chen2018metrics,arvanitidis2018latent}. By choosing $f$ to be the decoder function $f_D$, latent codes that are close w.r.t. $d_\Z$ induce curves of minimal energy in the ambient observation space (i.e., next state). This metric is closely related to the pullback metric (see \Cref{section: riemannian manifolds}), as shown by the following proposition.
\begin{proposition}
\label{prop: metric}
    Let $(\mathcal{M}_\Z, g_\Z)$ as defined above. Then $G_f(z) = J^T_f(z) J_f(z)$, for any $z \in \Z$.
\end{proposition}

Indeed, \Cref{prop: metric} shows us that $G_f$ is a pullback metric. Particularly $\Z$ and $J_f$ define the structure of the ambient observation space $\mathcal{X}$ (in our case, next state transitions). 

By choosing $f$ to be the decoder function $f_D$, the metric $G_{f_D}$ becomes stochastic, complicating analysis. Instead, as proposed and analyzed in \citet{arvanitidis2018latent}, we use the expected pullback metric $\E{G_\Z}$ as an approximation of the underlying stochastic metric. Similar to previous work on variational models, we use a normally distributed decoder to define the output. Using \Cref{prop: metric}, we have the following result (see Appendix for proof).

\begin{restatable}{theorem}{metricdecoderprop}[\citet{arvanitidis2018latent}]
\label{prop: decoder metric} Assume ${f_D(\cdot | z) \sim \mathcal{N}(\mu(z), \sigma(z) I)}$. Then
    \begin{align}
        \E[f_D(\cdot | z)]{G_{f_D}(z)} 
        =
        G_\mu(z) + G_{\sigma}(z),
        \label{eq: decoder metric}
    \end{align}
    where $G_\mu(z) = J_{\mu}^T(z) J_{\mu}(z)$ and $G_\sigma(z) = J_{\sigma}^T(z)J_{\sigma}(z)$.
\end{restatable}

Given an embedded latent space $\Z$, the expected metric in \Cref{eq: decoder metric} gives us a sense of the topology of the latent space manifold induced by $f_D$. The terms $G_\mu = J_{\mu}^T J_{\mu}$ and $G_\sigma = J_{\sigma}^TJ_{\sigma}$ are in fact the induced pullback metrics of $\mu$ and $\sigma$, respectively.

\subsection{Capturing Epistemic and Aleatoric Uncertainty}
\label{section: metric epistemic aleatoric}

The previously proposed encoder-decoder model induces a metric which captures the structure of the learned dynamics. However, the decoder variance, $\sigma(z)$, does not differentiate between aleatoric uncertainty (environment dynamics) and epistemic uncertainty (missing data).

We propose two methods to enrich the metric in \Cref{eq: decoder metric} in order to achieve a better estimate of uncertainty. First, by using an ensemble of $M$ decoder functions $\brk[c]*{f_{D, i}}_{i=1}^M$ trained using standard bootstrap techniques \citep{efron1982jackknife}, we capture the traditional epistemic uncertainty of the decoder parameters. Second, to correctly distinguish epistemic and aleatoric uncertainty, we add a latent forward function to our previously proposed variational model. Specifically, our latent model consists of an encoder $E: \mathcal{S} \times \A \mapsto \mathcal{B}(\Z)$, forward model $f_F: \mathcal{Z} \mapsto \mathcal{B}(\mathcal{X})$ and decoder functions $f_{D, i}: \mathcal{X} \mapsto \mathcal{B}(\mathcal{S})$ such that $P(\cdot | s, a) = f_{D, i}(\cdot | x)$, and $x \sim f_F(\cdot | E(s,a))$. This structure enables us to capture the aleatoric uncertainty under the forward transition model $f_F$, and the epistemic uncertainty using the decoders $f_{D, i}$. That is, for $f_F(\cdot | z) \sim \mathcal{N}(\mu_F(z), \sigma_F(z) I)$, the variance model $\sigma_F(z)$ captures the stochasticity in model dynamics. This decomposition is also helpful whenever one wants to train an agent in latent space (e.g., for planning \citet{schrittwieser2020mastering})

Next, we turn to analyze the pullback metric induced by the proposed forward transition model. As both $f_F$ and $\brk[c]{f_{D, i}}_{i=1}^m$ are stochastic (capturing epistemic and aleatoric uncertainty), the result of \Cref{prop: decoder metric} cannot be directly applied to their composition. The following proposition provides an analytical expression for the expected pullback metric of a sampled next state and a uniformly sampled decoder (the proof is given in the appendix).

\begin{figure*}[t!]
\captionsetup[subfigure]{labelformat=empty}
\begin{subfigure}{.059 \textwidth}
    \centering
    \includegraphics[width=\textwidth]{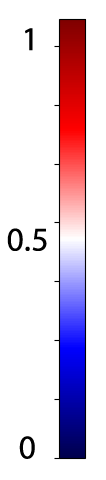}
    \caption{ }
\end{subfigure}
\begin{subfigure}{.37\textwidth}
    \centering
    \includegraphics[width=\textwidth]{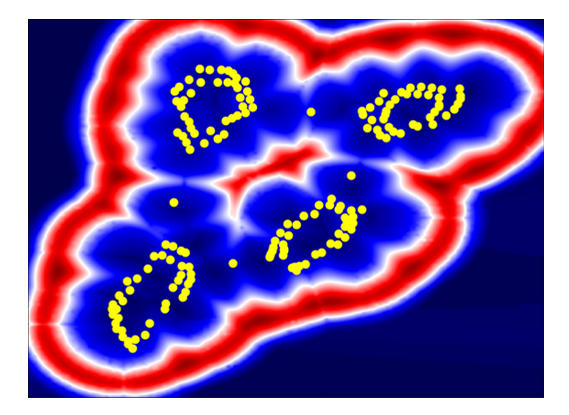}
    \caption{(a) $\sqrt{\det\brk*{G_{f_D}}}$}
\end{subfigure}
\begin{subfigure}{.275\textwidth}
    \centering
    \includegraphics[width=\textwidth]{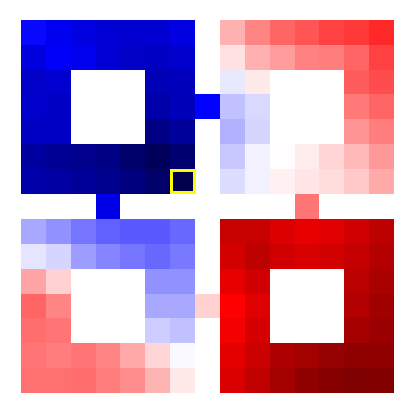}
    \caption{(b) Latent Geodesic Distance}
\end{subfigure}
\begin{subfigure}{.275\textwidth}
    \centering
    \includegraphics[width=\textwidth]{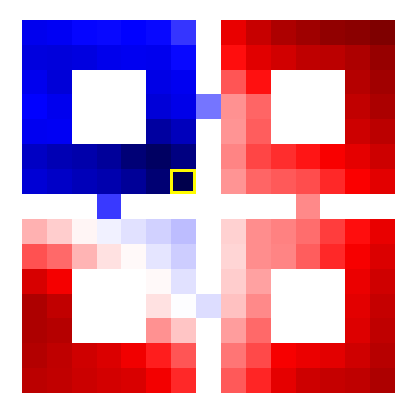}
    \caption{(c) Latent Euclidean Distance}
\end{subfigure}
\caption{\label{fig: four rooms} \small \textbf{(a)} The latent space (yellow markers) of the grid world environment and the geometric volume measure of the decoder-induced metric (background). \textbf{(b, c)} The geodesic distance of the decoder-induced submanifold and the Euclidean distance of latent states, as viewed in ambient space. All distances are calculated w.r.t. the yellow marked state. \textbf{Note}: colors in (a), which measure magnitude, are unrelated to colors in (b,c), which measure distance to the yellow marked state.}
\vskip -0.1in
\end{figure*} 

\begin{restatable}{theorem}{metriccompprop}
\label{prop: metric comp}
Assume ${f_F(\cdot | z) \sim \mathcal{N}(\mu_F(z), \sigma_F(z) I)}$, ${f_{D, i}(\cdot | x) \sim \mathcal{N}(\mu_D^i(x), \sigma_D^i(x) I)}$, ${U \sim \text{Unif}\brk[c]*{1, \hdots, M}}$. Then, the expected pullback metric of the composite function $(f_{D, U} \circ f_F)$ is given by
    \begin{align*}
        \E[P(f_{D, U} \circ f_F)]{G_{f_{D, U} \circ f_F}(z)} 
        =
        J_{\mu_F}^T(z) \overline{G}_{f_D}(z)J_{\mu_F}(z) 
        +
        J_{\sigma_F}^T(z) \diag{\overline{G}_{f_D}(z)}J_{\sigma_F}(z),
    \end{align*}
    where $\overline{G}_{f_D}(z)
    = 
    \frac{1}{M}\sum_{i=1}^M 
    \E[x \sim F(\cdot | z)]{ J_{\mu_D^i}^T(x) J_{\mu_D^i}(x) + J_{\sigma_D^i}^T(x) J_{\sigma_D^i}(x)}.
    $
\end{restatable}

Unlike the metric in \Cref{eq: decoder metric}, the composite metric distorts the decoder metric with Jacobian matrices of the forward model statistics. It takes into account both the aleatoric and epistemic uncertainty through the forward model as well as ensemble of decoders. As a special case we note the metric for the case of deterministic model dynamics.

\begin{corollary}
    Assume deterministic model dynamics, i.e., $x = f_F(z)$, and without loss of generality assume $f_F \equiv I$. Then, the expected pullback metric of \Cref{prop: metric comp} is given by
    $
    \E[]{G_{f_{D,U} \circ f_F}(z)}
    =
    \frac{1}{M}\sum_{i=1}^M J_{\mu_D^i}^T(z) J_{\mu_D^i}(z)
    + 
    J_{\sigma_D^i}^T(z)  J_{\sigma_D^i}(z).
    $
\end{corollary}

\subsection{GELATO: Geometrically Enriched LATent model for Offline reinforcement learning}
\label{section: GELATO}

Having defined our metric, we are now ready to leverage it in a model based offline RL framework. 
Specifically, provided a dataset ${\D_n = \brk[c]*{\brk*{s_i, a_i, r_i, s'_i}}}_{i=1}^n$ we train the variational latent forward model described in the previous section. 

\Cref{alg: gelato} presents GELATO, our proposed approach. In GELATO, we estimate model error by measuring the distance of a new sample to the data manifold. We construct the reward-penalized MDP for which the error acts as a pessimistic regularizer. Finally, we train an RL agent over the pessimistic MDP with transition $\hat{P}(\cdot | s, a)$ and reward $r(s,a) - \lambda U(s,a)$. By achieving an improved estimate for model error the model-based pessimistic approach can significantly improve performance, as shown in the following section.

\section{Experiments}
\label{section: experiments}

This section is dedicated to quantitatively and qualitatively understand the benefits of our proposed metric and method. We validate two principal claims: \textbf{(1) Our metric captures inherent characteristics of model dynamics.} We demonstrate this by visualizing state-action geodesics of a toy grid world problem and a multi-agent autonomous driving task. We show that curves of minimum energy in ambient space indeed capture intrinsic properties of the problem. \textbf{(2) Our metric provides an improved OOD uncertainty estimate for offline RL.} We compare the traditionally used bootstrap ensemble method to our approach, which leverages our pullback metric in a nonparametric nearest neighbors approach. We also compare our method to simple use of Euclidean distances in latent space. We run extensive experiments on continuous control and autonomous driving benchmarks. We show that our metric achieves significantly improved performance in tasks for which geodesics are non-euclidean.

\begin{figure*}[t!]
\centering
\includegraphics[width=0.85\textwidth]{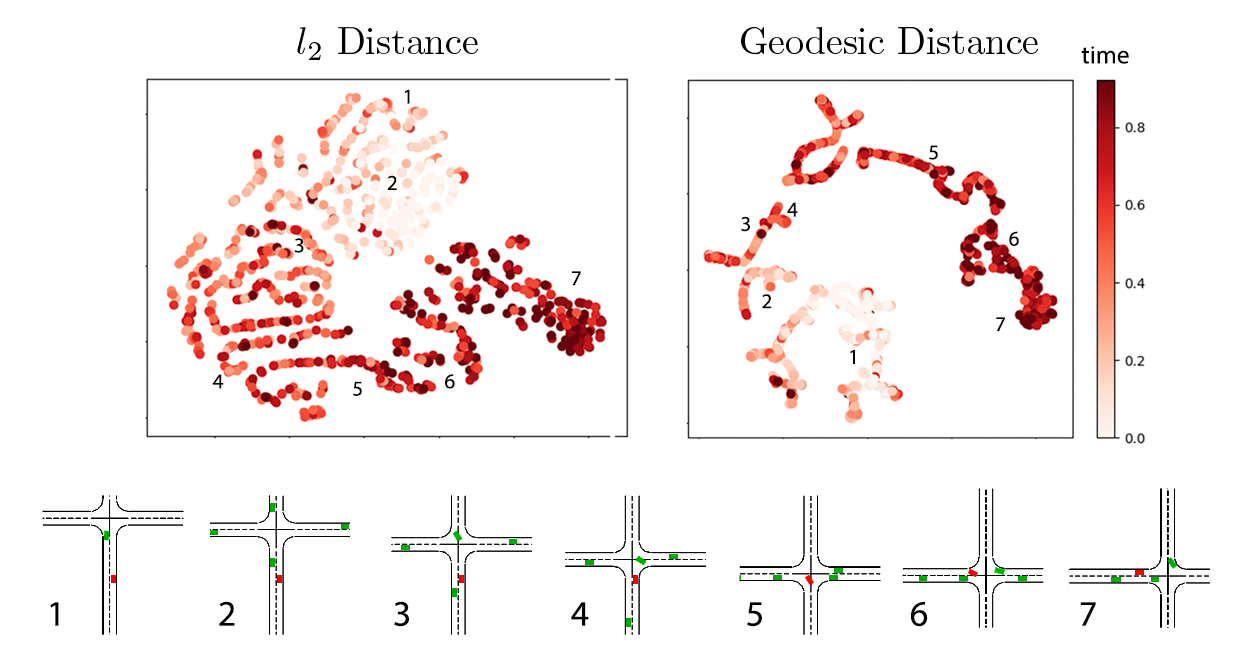}
\caption{\label{fig: highway embeddings} \small Plots show t-SNE embeddings generated in the intersection environment. Left plot depicts embeddings using Euclidean distances. Right plot depicts embeddings using geodesics distances which induce curves of minimum energy in ambient space. Colors correspond to the time in a trajectory (normalized w.r.t. longest trajectory). Numbering show mappings of a specific trajectory's states onto the embedded space as the controlled red car takes a left turn at an intersection (trajectory visualization shown under plots). }
\vskip -0.1in
\end{figure*} 

\subsection{Metric Visualization}
\label{sec: metric visualization}

\paragraph{Four Rooms.} To better understand the inherent structure of our metric, we constructed a grid-world environment for visualizing our proposed latent representation and metric. The ${15\times15}$ environment, as depicted in \Cref{fig: four rooms}, consists of four rooms, with impassable obstacles in their centers. The agent, residing at some position $(x, y) \in [-1,1]^2$ in the environment can take one of four actions: up, down, left, or right -- moving the agent $1, 2$ or $3$ steps (uniformly distributed) in that direction. We collected a dataset of $10000$ samples, taking random actions at random initializations of the environment. The ambient state space was represented by the position of the agent -- a vector of dimension~2, normalized to values in ${[-1, 1]}$. Finally, we trained a variational latent model with latent dimension ${d_\Z = 2}$. We used a standard encoder $z \sim \mathcal{N}(\mu_\theta(s), \sigma_\theta(s))$ and decoder $s' \sim \mathcal{N}(\mu_\phi(z), \sigma_\phi(z))$ represented by neural networks trained end-to-end using the evidence lower bound. We refer the reader to the appendix for an exhaustive description of the training procedure.

The latent space output of our model is depicted by yellow markers in \Cref{fig: four rooms}a. Indeed, the latent embedding consists of four distinctive clusters, structured in a similar manner as our grid-world environment. Interestingly, the distortion of the latent space accurately depicts an intuitive notion of distance between states. As such, rooms are distinctively separated, with fair distance between each cluster. States of pathways between rooms clearly separate the room clusters, forming a topology with four discernible bottlenecks.

In addition to the latent embedding, \Cref{fig: four rooms}a depicts the geometric volume measure
$\sqrt{\det\brk*{G_{f_D}}}$ of the trained pullback metric induced by $f_D$. This quantity demonstrates the effective geodesic distances between states in the decoder-induced submanifold. Indeed geodesics between data points to points outside of the data manifold (i.e., outside of the red region), attain high values as integrals over areas of high magnitude. In contrast, geodesics near the data manifold would low values. 

We visualize the decoder-induced geodesic distance and compare it to the latent Euclidean distance in Figures~\ref{fig: four rooms}b~and~\ref{fig: four rooms}c, respectively. The plots depict the normalized distances of all states to the state marked by a yellow square. Evidently, the geodesic distance captures a better notion of distance in the said environment, correctly exposing the ``land distance" in ambient space. As expected, states residing in the bottom-right room are farthest from the source state, as reaching them comprises of passing through at least two bottleneck states. In contrast, the latent Euclidean distance does not properly capture these geodesics, exhibiting a similar distribution of distances in other rooms. Nevertheless, both geodesic and Euclidean distances reveal the intrinsic topological structure of the environment, that of which is not captured by the extrinsic coordinates $(x,y) \in [-1, 1]^2$. Particularly, the state coordinates $(x,y)$ would wrongly assign short distances to states across impassible walls or obstacles,  i.e., measuring the ``air distance".


\begin{table*}[t!]
\caption{\label{table: results} \small Performance of GELATO and its variants in comparison to contemporary model-based methods on D4RL datasets. Scores correspond to the return, averaged over 5 seeds (standard deviation removed due to space constraints and is given in the appendix). Results for MOPO, MBPO, SAC, and imitation are taken from \citet{yu2020mopo}. Mean score of dataset added for reference. Bold scores show an improved score w.r.t other methods. }
\centering
\hspace*{-0.8cm}  
\begin{scriptsize}
\begin{tabular}{|c|ccc|ccc|ccc|}
\toprule
\multicolumn{1}{|c|}{} & \multicolumn{3}{c|}{\bf {\small Hopper}} & \multicolumn{3}{c|}{\bf {\small Walker2d} } & \multicolumn{3}{c|}{\bf {\small Halfcheetah} }\\
\midrule
\bf {\small Method}     & \bf {\small Rand}   & \bf {\small Med}   & \bf {\small Med-Expert} & \bf {\small Rand}   & \bf {\small Med}   & \bf {\small Med-Expert} & \bf {\small Rand}   & \bf {\small Med}   & \bf {\small Med-Expert} \\ \hline
Data Score & 299   & 1021    & 1849   & 1  & 498 & 1062   & -303  & 3945  & 8059    \\ \hline
GELATO (metric) & \bf {685 }   & \bf {1676 }   & 574    & \bf 412    & \bf {1269 }   & \bf {1515 }   &  2560    & \bf {5168 }   & \textbf{7989}   \\ \hline
GELATO ($\ell_2$) & 544    & 1320    & 815    & 388    & 312    & 1255   & 512   &  4096    & 7304    \\ \hline
MOPO (bootstrap) & \bf {677 }   & 1202   & 1063   & \bf {396 } & 518  & 1296   & \bf{4114} & 4974  & 7594    \\ \hline
MBPO & 444    & 457    & {2105 }  & 251   & 370    & 222   & 3527   & 3228   & 907   \\ \hline
SAC & 664   & 325   & 1850  & 120   & 27   & -2  & 3502   & -839   & -78  \\ \hline
Imitation & 615   & 1234   & \bf 3907   &  47 & 193 & 329   &  -41 & 4201 & 4164   \\ \hline 
\bottomrule
\end{tabular}
\end{scriptsize}
\end{table*}%

\paragraph{Intersection.} We visualized our metric in the intersection environment proposed in \citet{highway-env}. \Cref{fig: highway embeddings} compares the Euclidean and geodesic distances of a partially trained agent. Unlike the previous toy example, to visualize the inherent manifolds we used t-SNE \citep{van2008visualizing} projections computed with Euclidean distance and compared them to the projection computed with geodesic distance, i.e., curves of minimum energy in ambient space (\Cref{definition: distance}). Indeed, the geodesics captured the inherent structure of the environment, whereas Euclidean distances only managed to capture general clusters. These suggest that Euclidean distance might not be representative for measuring distance in latent space, as will also become evident by our experiments in the following subsections.

\subsection{Datasets and Implementation Details} We used D4RL \citep{fu2020d4rl} and the autonomous vehicle environments highway-env \citep{highway-env} as benchmarks for all of our experiments. We tested GELATO on three Mujoco \citep{todorov2012mujoco} environments (Hopper, Walker2d, Halfcheetah) on datasets generated by a single policy and a mixture of two policies. Specifically, we used datasets generated by a random agent (1M samples), a partially trained agent, i.e, medium agent (1M samples), and a mixture of partially trained and expert agents (2M samples). For autonomous driving, we tested GELATO on four environments (Highway, Roundabout, Intersection, Racetrack), on datasets containing five episodes generated by a partially trained agent. We also tested a faster ($\times 15$ speedup) variant of the Highway environment, as well as a harder instantiation of the Intersection environment in which the number of cars was tripled (further details can be found in the appendix).

We trained our variational latent model in two phases. First, the model was fully trained using a calibrated Gaussian decoder \citep{rybkin2020simple}. Specifically, a maximum-likelihood estimate of the variance was used
${
    \sigma^* = \text{MSE}(\mu, \hat{\mu}) \in \arg\max_\sigma \mathcal{N}(\hat{\mu} | \mu, \sigma^2 I).
}$
Then, in the second stage we fit the variance decoder networks. Hyperparameters for training are found in the appendix.

In order to practically estimate the geodesic distance in \Cref{alg: gelato}, we defined a parametric curve in latent space and used gradient descent to minimize the curve's energy. 
The resulting curve and pullback metric were then used to calculate the geodesic distance by a numerical estimate of the curve length (See Appendix for an exhaustive overview of the estimation method). 

We used FAISS \citep{johnson2019billion} for efficient GPU-based $k$-nearest neighbors calculation. We set $k=5$ neighbors for the penalized reward (\Cref{eq: nearest neighbors}). Finally, we used a variant of Soft Learning, as proposed by \citet{yu2020mopo} as our RL algorithm for the continuous control benchmarks, and PPO \citep{schulman2017proximal} for the autonomous driving tasks. All agents were trained for 1M steps (for continuous control benchmarks) and 350K steps (for the driving benchmarks), using a single GPU (RTX 2080), and averaged over 5 seeds (see Appendix for more details).

\begin{figure*}[t!]
\captionsetup[subfigure]{labelformat=empty}
\centering
\begin{subfigure}{.31\textwidth}
    \centering
    \includegraphics[width=\textwidth]{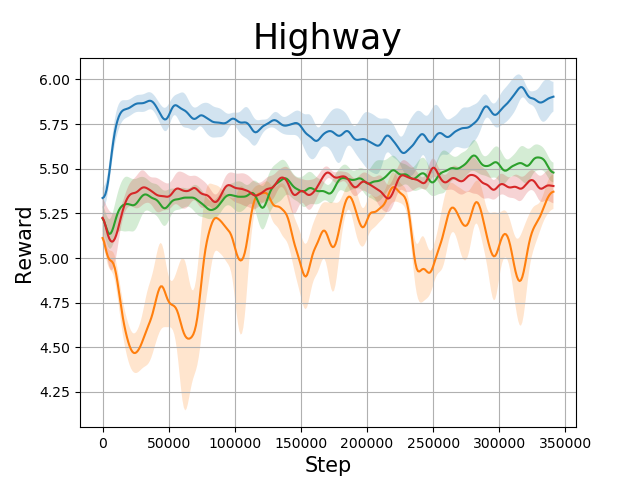}
\end{subfigure}
\begin{subfigure}{.31\textwidth}
    \centering
    \includegraphics[width=\textwidth]{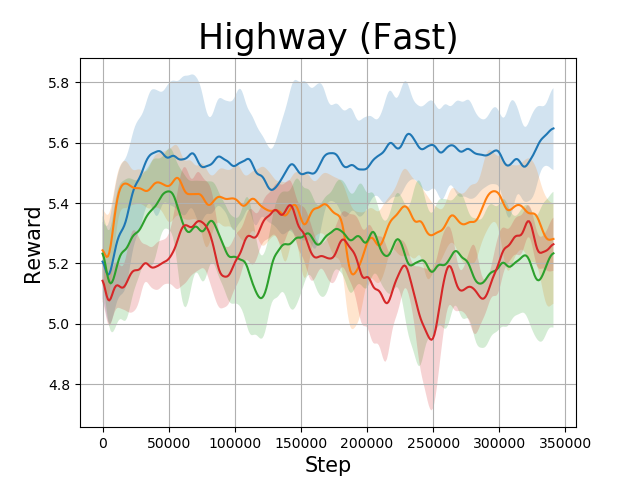}
\end{subfigure}
\begin{subfigure}{.31\textwidth}
    \centering
    \includegraphics[width=\textwidth]{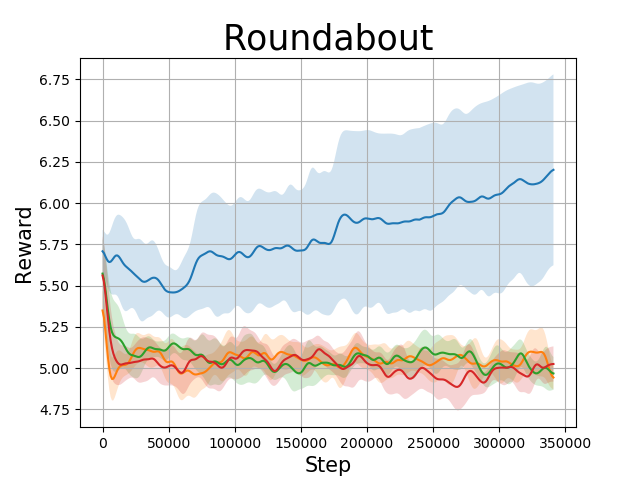}
\end{subfigure}
\begin{subfigure}{.31\textwidth}
    \centering
    \includegraphics[width=\textwidth]{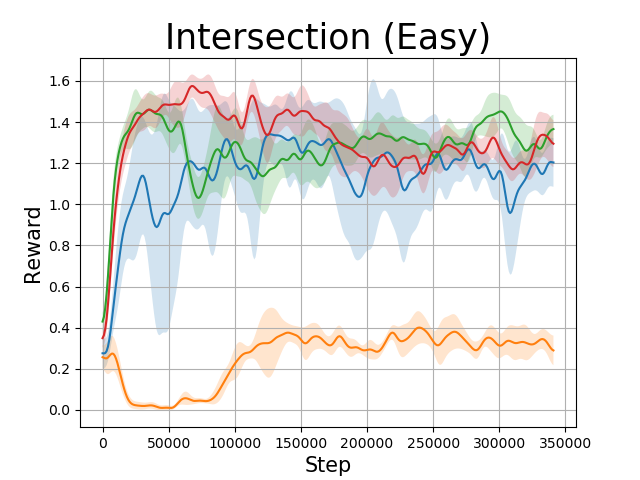}
\end{subfigure}
\begin{subfigure}{.31\textwidth}
    \centering
    \includegraphics[width=\textwidth]{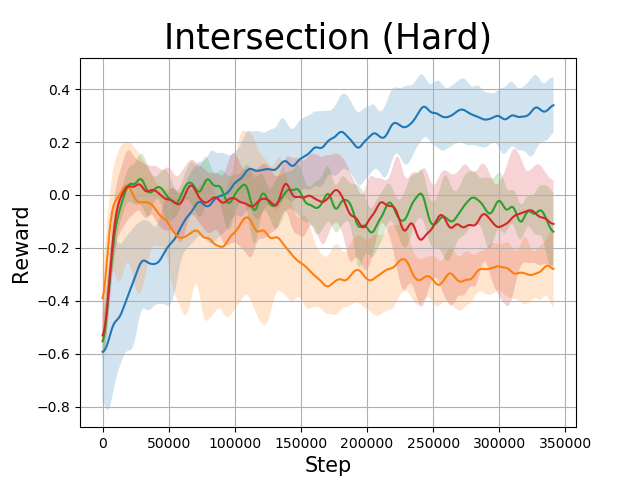}
\end{subfigure}
\begin{subfigure}{.31\textwidth}
    \centering
    \includegraphics[width=\textwidth]{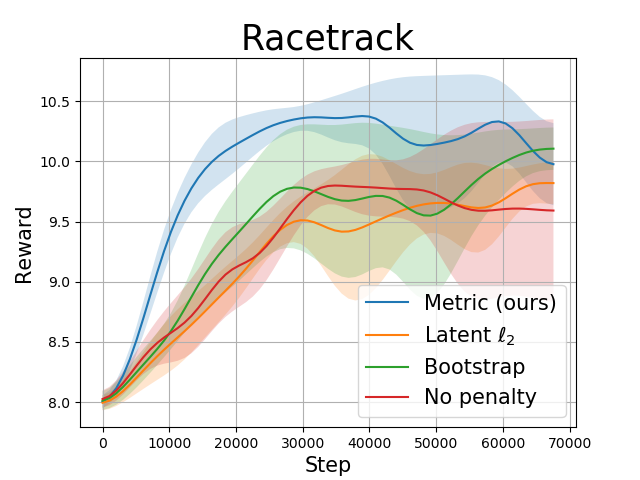}
\end{subfigure}
\caption{\label{fig: highway env results} \small Performance of GELATO with different uncertainty estimation methods on the highway-env benchmarks. Results suggest our induced semiparametric distance is an effective penalty for offline RL. }
\vskip -0.1in
\end{figure*} 

\subsection{Results}
\textbf{D4RL.} Results for the continuous control domains are shown in \Cref{table: results}. We performed experiments to analyze GELATO on various continuous control datasets. We compared GELATO to contemporary model-based offline RL approaches; namely, MOPO \citep{yu2020mopo} and MBPO \citep{janner2019trust}, as well as the standard baselines of SAC \citep{haarnoja2018soft} and imitation (behavioral cloning, \citet{bain1995framework,ross2011reduction}). 

We found performance increase on most domains, and most significantly in the medium domain, i.e., partially trained agent.  Since the medium dataset contained average behavior, combining the nonparametric nearest-neighbor uncertainty method with the bootstrap of decoders benefited the agent's overall performance. In addition, GELATO with the latent $\ell_2$ distance metric performed well on many of the benchmarks. We conjecture this is due to the inherent Euclidean nature of the continuous control benchmarks. Unlike the embedding for the autonomous driving benchmarks (\Cref{fig: highway embeddings}), we found the D4RL data to project similarly when $\ell_2$ and geodesic distances were used (we provide plots of these embeddings in the Appendix).

\textbf{Highway-Env.} \Cref{fig: highway env results} shows results for the autonomous driving benchmarks in highway-env. In contrast to the continuous control benchmarks, we found a significant improvement of our metric on the autonomous driving benchmarks compared to standard uncertainty estimation methods as well as the latent Euclidean distance. We credit this improvement to the non-euclidean nature of the environments, as previously described in \Cref{fig: highway embeddings}. While Euclidean distances were useful in the Mujoco environments, they performed distinctly worse in the autonomous driving environments.

Our results emphasize the importance of OOD uncertainty estimations methods in reinforcement learning on various types of datasets. While robotic control tasks provided useful insights, they did not capture the non-euclidean nature inherent in alternative tasks, such as autonomous driving.

\section{Related Work}

\textbf{Offline Reinforcement Learning.} The field of offline RL has recently received much attention as several algorithmic approaches were able to surpass standard off-policy algorithms. Value-based online algorithms do not perform well under highly off-policy batch data \citep{fujimoto2019off,kumar2019stabilizing,fu2019diagnosing,fedus2020revisiting,agarwal2020optimistic}, much due to issues with bootstrapping from out-of-distribution (OOD) samples. These issues become more prominent in the offline setting, as new samples cannot be acquired.

Several works on offline RL have shown improved performance on standard continuous control benchmarks \citep{laroche2019safe,kumar2019stabilizing,fujimoto2019off,chen2020bail,swazinna2020overcoming,kidambi2020morel,yu2020mopo}. This work focused on model-based approaches \citep{yu2020mopo,kidambi2020morel}, in which the agent is incentivized to remain close to areas of low uncertainty. Our work focused on controlling uncertainty estimation in high dimensional environments. Our methodology utilized recent success on the geometry of deep generative models \citep{arvanitidis2018latent,arvanitidis2020geometrically}, proposing an alternative approach to uncertainty estimation.

\textbf{Representation Learning.} Representation learning seeks to find an appropriate representation of data for performing a machine-learning task \citep{goodfellow2016deep}. Variational Auto Encoders \citep{kingma2013auto,rezende2014stochastic} have been a popular representation learning technique, particularly in unsupervised learning regimes \citep{chen2016variational,van2017neural,hsu2017unsupervised,serban2017hierarchical,engel2017latent,bojanowski2018optimizing,ding2020guided}, though also in supervised learning and reinforcement learning \citep{hausman2018learning,li2019multi,petangoda2019disentangled,hafner2019dream}. Particularly, variational models have been shown able to derive successful behaviors in high dimensional benchmarks \citep{hafner2020mastering}. 

Various representation techniques in reinforcement learning have also proposed to disentangle representation of both states \citep{engel2001learning,littman2002predictive,stooke2020decoupling,zhu2020masked}, and actions \citep{tennenholtz2019natural,chandak2019learning}. These allow for the abstraction of states and actions to significantly decrease computation requirements, by e.g., decreasing the effective dimensionality of the action space \citep{tennenholtz2019natural}. Unlike previous work, GELATO is focused on a semiparametric approach for uncertainty estimation, enhancing offline reinforcement learning performance.

\section{Discussion and Future Work}
\label{sec: negative results and future work}

This work presented a metric for model dynamics and its application to offline reinforcement learning. While our metric showed supportive evidence of improvement in model-based offline RL we note that it was significantly slower -- comparably, 5 times slower than using the decoder's variance for uncertainty estimation. The apparent slowdown in performance was mostly due to computation of the geodesic distance. Improvement in this area may utilize techniques for efficient geodesic estimation. 


We conclude by noting possible future applications of our work. In \Cref{sec: metric visualization} we demonstrated the inherent geometry our model had captured, its corresponding metric, and geodesics. Still, in this work we focused specifically on metrics related to the decoded state. In fact, a similar derivation to \Cref{prop: metric comp} could be applied to other modeled statistics, e.g., Q-values, rewards, future actions, and more. Each distinct statistic would induce its own unique metric w.r.t. its respective probability measure. Particularly, this concept may benefit a vast array of applications in continuous or large state and action spaces.


\bibliography{main.bib}

\begin{thebibliography}{68}
\providecommand{\natexlab}[1]{#1}
\providecommand{\url}[1]{\texttt{#1}}
\expandafter\ifx\csname urlstyle\endcsname\relax
  \providecommand{\doi}[1]{doi: #1}\else
  \providecommand{\doi}{doi: \begingroup \urlstyle{rm}\Url}\fi

\bibitem[Agarwal et~al.(2020)Agarwal, Schuurmans, and
  Norouzi]{agarwal2020optimistic}
Rishabh Agarwal, Dale Schuurmans, and Mohammad Norouzi.
\newblock An optimistic perspective on offline reinforcement learning.
\newblock In \emph{International Conference on Machine Learning}, pages
  104--114. PMLR, 2020.

\bibitem[Arvanitidis et~al.(2018)Arvanitidis, Hansen, and
  Hauberg]{arvanitidis2018latent}
Georgios Arvanitidis, Lars~Kai Hansen, and S{\o}ren Hauberg.
\newblock Latent space oddity: On the curvature of deep generative models.
\newblock In \emph{6th International Conference on Learning Representations,
  ICLR 2018}, 2018.

\bibitem[Arvanitidis et~al.(2020)Arvanitidis, Hauberg, and
  Sch{\"o}lkopf]{arvanitidis2020geometrically}
Georgios Arvanitidis, S{\o}ren Hauberg, and Bernhard Sch{\"o}lkopf.
\newblock Geometrically enriched latent spaces.
\newblock \emph{arXiv preprint arXiv:2008.00565}, 2020.

\bibitem[Bain and Sammut(1995)]{bain1995framework}
Michael Bain and Claude Sammut.
\newblock A framework for behavioural cloning.
\newblock In \emph{Machine Intelligence 15}, pages 103--129, 1995.

\bibitem[Bojanowski et~al.(2018)Bojanowski, Joulin, Lopez-Pas, and
  Szlam]{bojanowski2018optimizing}
Piotr Bojanowski, Armand Joulin, David Lopez-Pas, and Arthur Szlam.
\newblock Optimizing the latent space of generative networks.
\newblock In \emph{International Conference on Machine Learning}, pages
  600--609, 2018.

\bibitem[Buckman et~al.(2020)Buckman, Gelada, and
  Bellemare]{buckman2020importance}
Jacob Buckman, Carles Gelada, and Marc~G Bellemare.
\newblock The importance of pessimism in fixed-dataset policy optimization.
\newblock \emph{arXiv preprint arXiv:2009.06799}, 2020.

\bibitem[Carmo(1992)]{carmo1992riemannian}
Manfredo Perdigao~do Carmo.
\newblock \emph{Riemannian geometry}.
\newblock Birkh{\"a}user, 1992.

\bibitem[Chandak et~al.(2019)Chandak, Theocharous, Kostas, Jordan, and
  Thomas]{chandak2019learning}
Yash Chandak, Georgios Theocharous, James Kostas, Scott Jordan, and Philip
  Thomas.
\newblock Learning action representations for reinforcement learning.
\newblock In \emph{International Conference on Machine Learning}, pages
  941--950, 2019.

\bibitem[Chen et~al.(2018)Chen, Klushyn, Kurle, Jiang, Bayer, and
  Smagt]{chen2018metrics}
Nutan Chen, Alexej Klushyn, Richard Kurle, Xueyan Jiang, Justin Bayer, and
  Patrick Smagt.
\newblock Metrics for deep generative models.
\newblock In \emph{International Conference on Artificial Intelligence and
  Statistics}, pages 1540--1550. PMLR, 2018.

\bibitem[Chen et~al.(2019)Chen, Ferroni, Klushyn, Paraschos, Bayer, and van~der
  Smagt]{chen2019fast}
Nutan Chen, Francesco Ferroni, Alexej Klushyn, Alexandros Paraschos, Justin
  Bayer, and Patrick van~der Smagt.
\newblock Fast approximate geodesics for deep generative models.
\newblock In \emph{International Conference on Artificial Neural Networks},
  pages 554--566. Springer, 2019.

\bibitem[Chen et~al.(2020{\natexlab{a}})Chen, Klushyn, Ferroni, Bayer, and
  van~der Smagt]{chen2020learning}
Nutan Chen, Alexej Klushyn, Francesco Ferroni, Justin Bayer, and Patrick
  van~der Smagt.
\newblock Learning flat latent manifolds with vaes.
\newblock 2020{\natexlab{a}}.

\bibitem[Chen et~al.(2016)Chen, Kingma, Salimans, Duan, Dhariwal, Schulman,
  Sutskever, and Abbeel]{chen2016variational}
Xi~Chen, Diederik~P Kingma, Tim Salimans, Yan Duan, Prafulla Dhariwal, John
  Schulman, Ilya Sutskever, and Pieter Abbeel.
\newblock Variational lossy autoencoder.
\newblock \emph{arXiv preprint arXiv:1611.02731}, 2016.

\bibitem[Chen et~al.(2020{\natexlab{b}})Chen, Zhou, Wang, Wang, Wu, and
  Ross]{chen2020bail}
Xinyue Chen, Zijian Zhou, Zheng Wang, Che Wang, Yanqiu Wu, and Keith Ross.
\newblock Bail: Best-action imitation learning for batch deep reinforcement
  learning.
\newblock \emph{Advances in Neural Information Processing Systems}, 33,
  2020{\natexlab{b}}.

\bibitem[Ding et~al.(2020)Ding, Xu, Xu, Parmar, Yang, Welling, and
  Tu]{ding2020guided}
Zheng Ding, Yifan Xu, Weijian Xu, Gaurav Parmar, Yang Yang, Max Welling, and
  Zhuowen Tu.
\newblock Guided variational autoencoder for disentanglement learning.
\newblock In \emph{Proceedings of the IEEE/CVF Conference on Computer Vision
  and Pattern Recognition}, pages 7920--7929, 2020.

\bibitem[Dinh et~al.(2014)Dinh, Krueger, and Bengio]{dinh2014nice}
Laurent Dinh, David Krueger, and Yoshua Bengio.
\newblock Nice: Non-linear independent components estimation.
\newblock \emph{arXiv preprint arXiv:1410.8516}, 2014.

\bibitem[Efron(1982)]{efron1982jackknife}
Bradley Efron.
\newblock \emph{The jackknife, the bootstrap and other resampling plans}.
\newblock SIAM, 1982.

\bibitem[Engel et~al.(2017)Engel, Hoffman, and Roberts]{engel2017latent}
Jesse Engel, Matthew Hoffman, and Adam Roberts.
\newblock Latent constraints: Learning to generate conditionally from
  unconditional generative models.
\newblock \emph{arXiv preprint arXiv:1711.05772}, 2017.

\bibitem[Engel and Mannor(2001)]{engel2001learning}
Yaakov Engel and Shie Mannor.
\newblock Learning embedded maps of markov processes.
\newblock In \emph{Proceedings of the Eighteenth International Conference on
  Machine Learning}, pages 138--145, 2001.

\bibitem[Ernst et~al.(2005)Ernst, Geurts, and Wehenkel]{ernst2005tree}
Damien Ernst, Pierre Geurts, and Louis Wehenkel.
\newblock Tree-based batch mode reinforcement learning.
\newblock \emph{Journal of Machine Learning Research}, 6\penalty0
  (Apr):\penalty0 503--556, 2005.

\bibitem[Fathabadi et~al.(2021)Fathabadi, Seyedian, and
  Malekian]{fathabadi2021comparison}
Aboalhasan Fathabadi, Seyed~Morteza Seyedian, and Arash Malekian.
\newblock Comparison of bayesian, k-nearest neighbor and gaussian process
  regression methods for quantifying uncertainty of suspended sediment
  concentration prediction.
\newblock \emph{Science of The Total Environment}, page 151760, 2021.

\bibitem[Fedus et~al.(2020)Fedus, Ramachandran, Agarwal, Bengio, Larochelle,
  Rowland, and Dabney]{fedus2020revisiting}
William Fedus, Prajit Ramachandran, Rishabh Agarwal, Yoshua Bengio, Hugo
  Larochelle, Mark Rowland, and Will Dabney.
\newblock Revisiting fundamentals of experience replay.
\newblock In \emph{International Conference on Machine Learning}, pages
  3061--3071. PMLR, 2020.

\bibitem[Fonteneau et~al.(2013)Fonteneau, Murphy, Wehenkel, and
  Ernst]{fonteneau2013batch}
Raphael Fonteneau, Susan~A Murphy, Louis Wehenkel, and Damien Ernst.
\newblock Batch mode reinforcement learning based on the synthesis of
  artificial trajectories.
\newblock \emph{Annals of operations research}, 208\penalty0 (1):\penalty0
  383--416, 2013.

\bibitem[Fu et~al.(2019)Fu, Kumar, Soh, and Levine]{fu2019diagnosing}
Justin Fu, Aviral Kumar, Matthew Soh, and Sergey Levine.
\newblock Diagnosing bottlenecks in deep q-learning algorithms.
\newblock In \emph{International Conference on Machine Learning}, pages
  2021--2030, 2019.

\bibitem[Fu et~al.(2020)Fu, Kumar, Nachum, Tucker, and Levine]{fu2020d4rl}
Justin Fu, Aviral Kumar, Ofir Nachum, George Tucker, and Sergey Levine.
\newblock D4rl: Datasets for deep data-driven reinforcement learning.
\newblock \emph{arXiv preprint arXiv:2004.07219}, 2020.

\bibitem[Fujimoto et~al.(2019)Fujimoto, Meger, and Precup]{fujimoto2019off}
Scott Fujimoto, David Meger, and Doina Precup.
\newblock Off-policy deep reinforcement learning without exploration.
\newblock In \emph{International Conference on Machine Learning}, pages
  2052--2062. PMLR, 2019.

\bibitem[Gal and Ghahramani(2016)]{gal2016dropout}
Yarin Gal and Zoubin Ghahramani.
\newblock Dropout as a bayesian approximation: Representing model uncertainty
  in deep learning.
\newblock In \emph{international conference on machine learning}, pages
  1050--1059, 2016.

\bibitem[Goodfellow et~al.(2016)Goodfellow, Bengio, Courville, and
  Bengio]{goodfellow2016deep}
Ian Goodfellow, Yoshua Bengio, Aaron Courville, and Yoshua Bengio.
\newblock \emph{Deep learning}, volume~1.
\newblock MIT press Cambridge, 2016.

\bibitem[Haarnoja et~al.(2018)Haarnoja, Zhou, Abbeel, and
  Levine]{haarnoja2018soft}
Tuomas Haarnoja, Aurick Zhou, Pieter Abbeel, and Sergey Levine.
\newblock Soft actor-critic: Off-policy maximum entropy deep reinforcement
  learning with a stochastic actor.
\newblock In \emph{International Conference on Machine Learning}, pages
  1861--1870. PMLR, 2018.

\bibitem[Hafner et~al.(2019)Hafner, Lillicrap, Ba, and
  Norouzi]{hafner2019dream}
Danijar Hafner, Timothy Lillicrap, Jimmy Ba, and Mohammad Norouzi.
\newblock Dream to control: Learning behaviors by latent imagination.
\newblock In \emph{International Conference on Learning Representations}, 2019.

\bibitem[Hafner et~al.(2020)Hafner, Lillicrap, Norouzi, and
  Ba]{hafner2020mastering}
Danijar Hafner, Timothy Lillicrap, Mohammad Norouzi, and Jimmy Ba.
\newblock Mastering atari with discrete world models.
\newblock \emph{arXiv preprint arXiv:2010.02193}, 2020.

\bibitem[Hausman et~al.(2018)Hausman, Springenberg, Wang, Heess, and
  Riedmiller]{hausman2018learning}
Karol Hausman, Jost~Tobias Springenberg, Ziyu Wang, Nicolas Heess, and Martin
  Riedmiller.
\newblock Learning an embedding space for transferable robot skills.
\newblock In \emph{International Conference on Learning Representations}, 2018.

\bibitem[Hsu et~al.(2017)Hsu, Zhang, and Glass]{hsu2017unsupervised}
Wei-Ning Hsu, Yu~Zhang, and James Glass.
\newblock Unsupervised learning of disentangled and interpretable
  representations from sequential data.
\newblock In \emph{Advances in neural information processing systems}, pages
  1878--1889, 2017.

\bibitem[Janner et~al.(2019)Janner, Fu, Zhang, and Levine]{janner2019trust}
Michael Janner, Justin Fu, Marvin Zhang, and Sergey Levine.
\newblock When to trust your model: Model-based policy optimization.
\newblock \emph{arXiv preprint arXiv:1906.08253}, 2019.

\bibitem[Jin et~al.(2020)Jin, Yang, and Wang]{jin2020pessimism}
Ying Jin, Zhuoran Yang, and Zhaoran Wang.
\newblock Is pessimism provably efficient for offline rl?
\newblock \emph{arXiv preprint arXiv:2012.15085}, 2020.

\bibitem[Johnson et~al.(2019)Johnson, Douze, and J{\'e}gou]{johnson2019billion}
Jeff Johnson, Matthijs Douze, and Herv{\'e} J{\'e}gou.
\newblock Billion-scale similarity search with gpus.
\newblock \emph{IEEE Transactions on Big Data}, 2019.

\bibitem[Kalatzis et~al.(2020)Kalatzis, Eklund, Arvanitidis, and
  Hauberg]{kalatzis2020variational}
Dimitris Kalatzis, David Eklund, Georgios Arvanitidis, and S{\o}ren Hauberg.
\newblock Variational autoencoders with riemannian brownian motion priors.
\newblock In \emph{Proceedings of the 37th International Conference on Machine
  Learning (ICML)}. PMLR, 2020.

\bibitem[Kidambi et~al.(2020)Kidambi, Rajeswaran, Netrapalli, and
  Joachims]{kidambi2020morel}
Rahul Kidambi, Aravind Rajeswaran, Praneeth Netrapalli, and Thorsten Joachims.
\newblock Morel: Model-based offline reinforcement learning.
\newblock \emph{arXiv preprint arXiv:2005.05951}, 2020.

\bibitem[Kingma and Ba(2014)]{kingma2014adam}
Diederik~P Kingma and Jimmy Ba.
\newblock Adam: A method for stochastic optimization.
\newblock \emph{arXiv preprint arXiv:1412.6980}, 2014.

\bibitem[Kingma and Welling(2013)]{kingma2013auto}
Diederik~P Kingma and Max Welling.
\newblock Auto-encoding variational bayes.
\newblock \emph{arXiv preprint arXiv:1312.6114}, 2013.

\bibitem[Klushyn et~al.(2019)Klushyn, Chen, Kurle, Cseke, and van~der
  Smagt]{klushyn2019learning}
Alexej Klushyn, Nutan Chen, Richard Kurle, Botond Cseke, and Patrick van~der
  Smagt.
\newblock Learning hierarchical priors in vaes.
\newblock In \emph{Advances in Neural Information Processing Systems}, pages
  2870--2879, 2019.

\bibitem[Kumar et~al.(2019)Kumar, Fu, Soh, Tucker, and
  Levine]{kumar2019stabilizing}
Aviral Kumar, Justin Fu, Matthew Soh, George Tucker, and Sergey Levine.
\newblock Stabilizing off-policy q-learning via bootstrapping error reduction.
\newblock In \emph{Advances in Neural Information Processing Systems}, pages
  11784--11794, 2019.

\bibitem[Laroche et~al.(2019)Laroche, Trichelair, and
  Des~Combes]{laroche2019safe}
Romain Laroche, Paul Trichelair, and Remi~Tachet Des~Combes.
\newblock Safe policy improvement with baseline bootstrapping.
\newblock In \emph{International Conference on Machine Learning}, pages
  3652--3661. PMLR, 2019.

\bibitem[Leurent(2018)]{highway-env}
Edouard Leurent.
\newblock An environment for autonomous driving decision-making.
\newblock \url{https://github.com/eleurent/highway-env}, 2018.

\bibitem[Levine et~al.(2020)Levine, Kumar, Tucker, and Fu]{levine2020offline}
Sergey Levine, Aviral Kumar, George Tucker, and Justin Fu.
\newblock Offline reinforcement learning: Tutorial, review, and perspectives on
  open problems.
\newblock \emph{arXiv preprint arXiv:2005.01643}, 2020.

\bibitem[Li et~al.(2019)Li, Wu, Jun, and Ammar]{li2019multi}
Minne Li, Lisheng Wu, WANG Jun, and Haitham~Bou Ammar.
\newblock Multi-view reinforcement learning.
\newblock In \emph{Advances in neural information processing systems}, pages
  1420--1431, 2019.

\bibitem[Liang et~al.(2018)Liang, Liaw, Nishihara, Moritz, Fox, Goldberg,
  Gonzalez, Jordan, and Stoica]{liang2018rllib}
Eric Liang, Richard Liaw, Robert Nishihara, Philipp Moritz, Roy Fox, Ken
  Goldberg, Joseph Gonzalez, Michael Jordan, and Ion Stoica.
\newblock Rllib: Abstractions for distributed reinforcement learning.
\newblock In \emph{International Conference on Machine Learning}, pages
  3053--3062. PMLR, 2018.

\bibitem[Littman and Sutton(2002)]{littman2002predictive}
Michael~L Littman and Richard~S Sutton.
\newblock Predictive representations of state.
\newblock In \emph{Advances in neural information processing systems}, pages
  1555--1561, 2002.

\bibitem[Osband et~al.(2018)Osband, Aslanides, and
  Cassirer]{osband2018randomized}
Ian Osband, John Aslanides, and Albin Cassirer.
\newblock Randomized prior functions for deep reinforcement learning.
\newblock \emph{Advances in Neural Information Processing Systems}, 31, 2018.

\bibitem[Petangoda et~al.(2019)Petangoda, Pascual-Diaz, Adam, Vrancx, and
  Grau-Moya]{petangoda2019disentangled}
Janith~C Petangoda, Sergio Pascual-Diaz, Vincent Adam, Peter Vrancx, and Jordi
  Grau-Moya.
\newblock Disentangled skill embeddings for reinforcement learning.
\newblock \emph{arXiv preprint arXiv:1906.09223}, 2019.

\bibitem[Puterman(2014)]{puterman2014markov}
Martin~L Puterman.
\newblock \emph{Markov decision processes: discrete stochastic dynamic
  programming}.
\newblock John Wiley \& Sons, 2014.

\bibitem[Rezende et~al.(2014)Rezende, Mohamed, and
  Wierstra]{rezende2014stochastic}
Danilo~Jimenez Rezende, Shakir Mohamed, and Daan Wierstra.
\newblock Stochastic backpropagation and approximate inference in deep
  generative models.
\newblock \emph{arXiv preprint arXiv:1401.4082}, 2014.

\bibitem[Riedmiller(2005)]{riedmiller2005neural}
Martin Riedmiller.
\newblock Neural fitted q iteration--first experiences with a data efficient
  neural reinforcement learning method.
\newblock In \emph{European Conference on Machine Learning}, pages 317--328.
  Springer, 2005.

\bibitem[Ross et~al.(2011)Ross, Gordon, and Bagnell]{ross2011reduction}
St{\'e}phane Ross, Geoffrey Gordon, and Drew Bagnell.
\newblock A reduction of imitation learning and structured prediction to
  no-regret online learning.
\newblock In \emph{Proceedings of the fourteenth international conference on
  artificial intelligence and statistics}, pages 627--635. JMLR Workshop and
  Conference Proceedings, 2011.

\bibitem[Rybkin et~al.(2020)Rybkin, Daniilidis, and Levine]{rybkin2020simple}
Oleh Rybkin, Kostas Daniilidis, and Sergey Levine.
\newblock Simple and effective vae training with calibrated decoders.
\newblock \emph{arXiv preprint arXiv:2006.13202}, 2020.

\bibitem[Schrittwieser et~al.(2020)Schrittwieser, Antonoglou, Hubert, Simonyan,
  Sifre, Schmitt, Guez, Lockhart, Hassabis, Graepel,
  et~al.]{schrittwieser2020mastering}
Julian Schrittwieser, Ioannis Antonoglou, Thomas Hubert, Karen Simonyan,
  Laurent Sifre, Simon Schmitt, Arthur Guez, Edward Lockhart, Demis Hassabis,
  Thore Graepel, et~al.
\newblock Mastering atari, go, chess and shogi by planning with a learned
  model.
\newblock \emph{Nature}, 588\penalty0 (7839):\penalty0 604--609, 2020.

\bibitem[Schulman et~al.(2017)Schulman, Wolski, Dhariwal, Radford, and
  Klimov]{schulman2017proximal}
John Schulman, Filip Wolski, Prafulla Dhariwal, Alec Radford, and Oleg Klimov.
\newblock Proximal policy optimization algorithms.
\newblock \emph{arXiv preprint arXiv:1707.06347}, 2017.

\bibitem[Serban et~al.(2017)Serban, Sordoni, Lowe, Charlin, Pineau, Courville,
  and Bengio]{serban2017hierarchical}
Iulian Serban, Alessandro Sordoni, Ryan Lowe, Laurent Charlin, Joelle Pineau,
  Aaron Courville, and Yoshua Bengio.
\newblock A hierarchical latent variable encoder-decoder model for generating
  dialogues.
\newblock In \emph{Proceedings of the AAAI Conference on Artificial
  Intelligence}, volume~31, 2017.

\bibitem[Stooke et~al.(2020)Stooke, Lee, Abbeel, and
  Laskin]{stooke2020decoupling}
Adam Stooke, Kimin Lee, Pieter Abbeel, and Michael Laskin.
\newblock Decoupling representation learning from reinforcement learning.
\newblock \emph{arXiv preprint arXiv:2009.08319}, 2020.

\bibitem[Swazinna et~al.(2020)Swazinna, Udluft, and
  Runkler]{swazinna2020overcoming}
Phillip Swazinna, Steffen Udluft, and Thomas Runkler.
\newblock Overcoming model bias for robust offline deep reinforcement learning.
\newblock \emph{arXiv preprint arXiv:2008.05533}, 2020.

\bibitem[Tennenholtz and Mannor(2019)]{tennenholtz2019natural}
Guy Tennenholtz and Shie Mannor.
\newblock The natural language of actions.
\newblock In \emph{International Conference on Machine Learning}, pages
  6196--6205, 2019.

\bibitem[Todorov et~al.(2012)Todorov, Erez, and Tassa]{todorov2012mujoco}
Emanuel Todorov, Tom Erez, and Yuval Tassa.
\newblock Mujoco: A physics engine for model-based control.
\newblock In \emph{2012 IEEE/RSJ International Conference on Intelligent Robots
  and Systems}, pages 5026--5033. IEEE, 2012.

\bibitem[Van Den~Oord et~al.(2017)Van Den~Oord, Vinyals, et~al.]{van2017neural}
Aaron Van Den~Oord, Oriol Vinyals, et~al.
\newblock Neural discrete representation learning.
\newblock In \emph{Advances in Neural Information Processing Systems}, pages
  6306--6315, 2017.

\bibitem[Van~der Maaten and Hinton(2008)]{van2008visualizing}
Laurens Van~der Maaten and Geoffrey Hinton.
\newblock Visualizing data using t-sne.
\newblock \emph{Journal of machine learning research}, 9\penalty0 (11), 2008.

\bibitem[Villa~Medina et~al.(2013)]{villa2013reliability}
Joe~Luis Villa~Medina et~al.
\newblock \emph{Reliability of classification and prediction in k-nearest
  neighbours}.
\newblock PhD thesis, Universitat Rovira i Virgili, 2013.

\bibitem[Wang et~al.(2020)Wang, Foster, and Kakade]{wang2020statistical}
Ruosong Wang, Dean~P Foster, and Sham~M Kakade.
\newblock What are the statistical limits of offline rl with linear function
  approximation?
\newblock \emph{arXiv preprint arXiv:2010.11895}, 2020.

\bibitem[Yu et~al.(2020)Yu, Thomas, Yu, Ermon, Zou, Levine, Finn, and
  Ma]{yu2020mopo}
Tianhe Yu, Garrett Thomas, Lantao Yu, Stefano Ermon, James Zou, Sergey Levine,
  Chelsea Finn, and Tengyu Ma.
\newblock Mopo: Model-based offline policy optimization.
\newblock \emph{arXiv preprint arXiv:2005.13239}, 2020.

\bibitem[Zanette(2020)]{zanette2020exponential}
Andrea Zanette.
\newblock Exponential lower bounds for batch reinforcement learning: Batch rl
  can be exponentially harder than online rl.
\newblock \emph{arXiv preprint arXiv:2012.08005}, 2020.

\bibitem[Zhu et~al.(2020)Zhu, Xia, Wu, Deng, Zhou, Qin, and Li]{zhu2020masked}
Jinhua Zhu, Yingce Xia, Lijun Wu, Jiajun Deng, Wengang Zhou, Tao Qin, and
  Houqiang Li.
\newblock Masked contrastive representation learning for reinforcement
  learning.
\newblock \emph{arXiv preprint arXiv:2010.07470}, 2020.

\end{thebibliography}
\bibliographystyle{plainnat}


\newpage

\onecolumn
\section*{Appendix}

\begin{table*}[t!]
\caption{\label{table: results2} \small Results of GELATO as presented in \Cref{table: results} with added std for each run, averaged over 5 seeds.}
\centering
\hspace*{-2.5cm}  
\begin{scriptsize}
\begin{tabular}{|c|ccc|ccc|ccc|}
\toprule
\multicolumn{1}{|c|}{} & \multicolumn{3}{c|}{\bf {\small Hopper}} & \multicolumn{3}{c|}{\bf {\small Walker2d} } & \multicolumn{3}{c|}{\bf {\small Halfcheetah} }\\
\midrule
\bf {\small Method}     & \bf {\small Rand}   & \bf {\small Med}   & \bf {\small Med-Expert} & \bf {\small Rand}   & \bf {\small Med}   & \bf {\small Med-Expert} & \bf {\small Rand}   & \bf {\small Med}   & \bf {\small Med-Expert} \\ \hline
Data Score & $299 \pm 200$   & $1021 \pm 314$    & $1849 \pm 1560$   & $1 \pm 6$  & $498 \pm 807$ & $1062 \pm 1576$   & $-303 \pm 79$  & $3945 \pm 494$  & $8059 \pm 4204$    \\ \hline
GELATO (metric) & $\mathbf{685 \pm 15}$   & $\mathbf{1676 \pm 223}$   & $574 \pm 16$    & $\mathbf{412 \pm 85}$    & $\mathbf{1269 \pm 549}$   & $\mathbf {1515 \pm 379}$   &  $2560 \pm 240$    & $\mathbf {5168 \pm 849 }$   & $\mathbf{7989 \pm 2790}$   \\ \hline
GELATO ($\ell_2$) & 544 $\pm$ 29   & 1320 $\pm$ 423    & 815 $\pm$ 153    & 388 $\pm$ 49    & 312 $\pm$ 189   & 1255 $\pm$ 310  & 512 $\pm$ 71   &  4096 $\pm$ 585    & 7304 $\pm$ 3780    \\ \hline
MOPO & \bf {677 $\pm$ 13}   & 1202 $\pm$ 400   & 1063 $\pm$ 193   & \bf {396 $\pm$ 76} & 518 $\pm$ 560 & 1296 $\pm$ 374   & \bf{4114 $\pm$ 312} & \bf 4974 $\pm$ 200 & \bf {7594 $\pm$ 4741}   \\ \hline
MBPO & 444 $\pm$ 193   & 457 $\pm$ 106   & {2105 $\pm$ 1113}  & 251 $\pm$ 235   & 370 $\pm$ 221   & 222 $\pm$ 99  & 3527 $\pm$ 487   & 3228 $\pm$ 2832   & 907 $\pm$ 1185  \\ \hline
SAC & 664   & 325   & 1850  & 120   & 27   & -2  & 3502   & -839   & -78  \\ \hline
Imitation & 615   & 1234   & \bf 3907   &  47 & 193 & 329   &  -41 & 4201 & 4164   \\ \hline 
\bottomrule
\end{tabular}
\end{scriptsize}
\end{table*}%

\begin{figure*}[t!]
\centering
\includegraphics[width=0.9\textwidth]{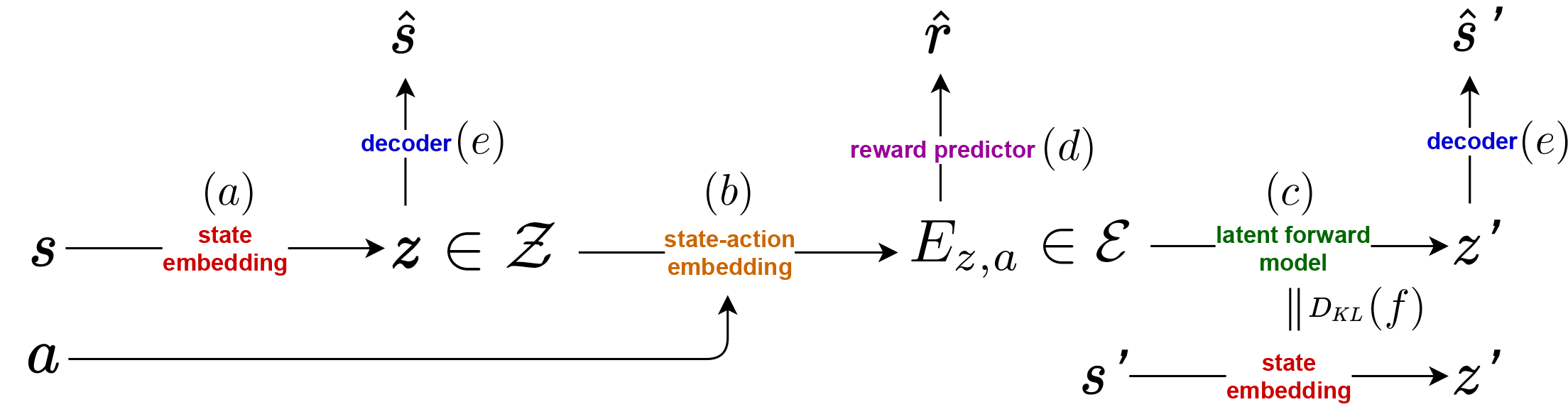}
\caption{\label{fig: graphical model} \small A graphical representation of our latent variable model. \textbf{(a)} The state $s$ is embedded via the state embedding function (i.e., approximate posterior) $z \sim q(\cdot | s)$. \textbf{(b)} The action and embedded state pass through an invertible embedding function $E$ to produce the state-action embedding $E_{z,a}$. \textbf{(c,d)} The state-action embedding is passed through a reward predictor and latent forward model, $\hat{r} \sim P(\cdot | E_{z,a})$ and $z' \sim P(\cdot | z, a)$, respectively. \textbf{(e)} The next latent state $z'$ is decoded back to observation space to generate $\hat{s}' \sim P(\cdot | z')$. \textbf{(f)} Finally, during training, the target state $s'$ is embedded and compared to $z'$ (by the KL-divergence term in \Cref{eq: elbo}), preserving the consistency of the latent space $\Z$. }
\end{figure*}

\section{Variational Latent Model}
\label{sec: model}

We begin by describing our variational forward model. The model, based on an encoder, latent forward function, and decoder framework assumes the underlying dynamics and reward are governed by a state-embedded latent space $\Z \subseteq \R^{d_\Z}$. The probability of a trajectory ${\tau = (s_0, a_0, r_0, \hdots, s_h, a_h, r_h)}$ is given by
\begin{align}
    P(\tau) = \int_{z_0, \hdots, z_h} &P(z_0)\prod_{i=0}^{h} P(s_i| z_i) \pi(a_i | s_i)P(r_i | E_{z_i, a_i})
    \prod_{j=1}^{h} P(z_j | E_{z_{j-1}, a_{j-1}}) dz_0 \hdots dz_h, \label{eq: likelihood}
\end{align}
where $E : \Z \times \A \mapsto \mathcal{E} \subseteq \R^{d_{\mathcal{E}}}$ is a deterministic, invertible embedding function which maps pairs $(z, a)$ to a state-action-embedded latent space $\mathcal{E}$. $E_{z,a}$ is thus a sufficient statistic of $(z, a)$. The proposed graphical model is depicted in \Cref{fig: graphical model}. We note that an extension to the partially observable setting replaces $s_t$ with $h_t = \brk*{s_0, a_0, \hdots, s_t}$, a sufficient statistic of the unknown state.

Maximizing the log-likelihood $\log P(\tau)$ is hard due to intractability of the integral in \Cref{eq: likelihood}. We therefore introduce the approximate posterior $q(z | s)$ and maximize the evidence lower bound.

To clear notations we define $E_{z_{-1}, a_{-1}} = 0$, so that we can rewrite the above expression as
\begin{align*}
     P(s_0, a_0, r_0, \hdots, s_h, a_h, r_h) 
     =
    \int_{z_0, \hdots, z_h} \prod_{i=0}^{h} P(s_i| z_i) \pi(a_i | s_i) P(r_i | E_{z_i, a_i}) P(z_j |E_{z_{-1}, a_{-1}})
\end{align*}
Introducing $q(z_i | s_i)$ we can write
\begin{align*}
    &\log
    \int_{z_0, \hdots, z_h} 
    \prod_{i=0}^{h} 
    \frac{q(z_i | s_i)}{q(z_i | s_i)}
    \prod_{i=0}^{h} 
    P(s_i| z_i)\pi(a_i | s_i)
    P(r_i | E_{z_i, a_i}) P(z_j |E_{z_{-1}, a_{-1}}) \\
    &\geq
    \int_{z_0, \hdots, z_h} 
    \prod_{i=0}^{h} 
    q(z_i | s_i)
    \brk*
    {
        \sum_{i=0}^{h} 
        \log
        \brk*
        {
            \frac{
                P(s_i| z_i)
                \pi(a_i | s_i)
                P(r_i | E_{z_i, a_i}) 
                P(z_j |E_{z_{-1}, a_{-1}})
            }
            {
                q(z_i | s_i)
            }
        }
    } \\
    &=
    \sum_{i=0}^{H}
    \E[q(z_i \mid s_i)]
    {
        \log
        \brk*{
            P(s_i \mid z_i)
            \pi(a_i \mid s_i)
            P(r_i | E_{z_i, a_i}) 
        }
    } \\
    &-
    \sum_{i=0}^{H-1}
    \E[q(z_i \mid s_i)]
    {
        \kld{q(z_{i+1} \mid s_{i+1})}{P(z_{i+1} | E_{z_i, a_i})}
    } \\
    &-
    \kld{q(z_0 \mid s_0)}{P(z_0)}.
\end{align*}

Hence,
\begin{align}
    &\sum_{i=0}^{h}
    \E[z_i \sim q(z_i \mid s_i)]
    {
        \log
        \brk*{
            P(s_i \mid z_i)
            \pi(a_i \mid s_i)
            P(r_i | E_{z_i, a_i})
        }
    } \nonumber \\
    &-
    \sum_{i=0}^{h-1}
    \E[z_i \sim q(z_i \mid s_i)]
    {
        \kld{q(z_{i+1} \mid s_{i+1})}{P(z_{i+1} | E_{z_i, a_i})}
    } \nonumber \\
    &-
    \kld{q(z_0 \mid s_0)}{P(z_0)}. \label{eq: elbo}
\end{align}

The distribution parameters of the approximate posterior $q(z | s)$, the likelihoods $P(s|z), \pi(a | s), P(r | E_{z,a})$, and the latent forward model $P(z' | E_{z,a})$ are represented by neural networks. The invertible embedding function $E$ is represented by an invertible neural network, e.g., affine coupling, commonly used for normalizing flows \citep{dinh2014nice}. Though various latent distributions have been proposed \citep{klushyn2019learning,kalatzis2020variational}, we found Gaussian parametric distributions to suffice for all of our model's functions. Particularly, we used two outputs for every distribution, representing the expectation $\mu$ and variance $\sigma$. All networks were trained end-to-end to maximize the evidence lower bound in \Cref{eq: elbo}.

\section{Specific Implementation Details}

\textbf{D4RL.} As a preprocessing step rewards were normalized to values between $[-1, 1]$. We trained our variational model with latent dimensions $\dim\brk*{\Z} = 32$ and ${\dim\brk*{\mathcal{E}} = \dim\brk*{\Z} + \dim\brk*{\A}}$. All domains were trained with the same hyperparameters. Specifically, we used a 2-layer Multi Layer Perceptron (MLP) to encode $\Z$, after which a 2-layer Affine Coupling (AC) \citep{dinh2014nice} was used to encode $\mathcal{E}$. We also used a 2-layer MLP for the forward, reward, and decoder models. All layers contained 256 hidden layers. 

The latent model was trained in two separate phases for 100k and 50k steps each by stochastic gradient descent and the ADAM optimizer \citep{kingma2014adam}. First, the model was fully trained using a calibrated Gaussian decoder \citep{rybkin2020simple}. Specifically, a maximum-likelihood estimate of the variance was used $\sigma^* = \text{MSE}(\mu, \hat{\mu}) \in \arg\max_\sigma \mathcal{N}(\hat{\mu} | \mu, \sigma^2 I)$. Finally, in the second stage we fit the variance decoder network. We found this process of to greatly improve convergence speed and accuracy, and mitigate posterior collapse. We used a minimum variance of $0.01$ for all of our stochastic models. 

\begin{figure}[t!]
\centering
\includegraphics[width=0.6\textwidth]{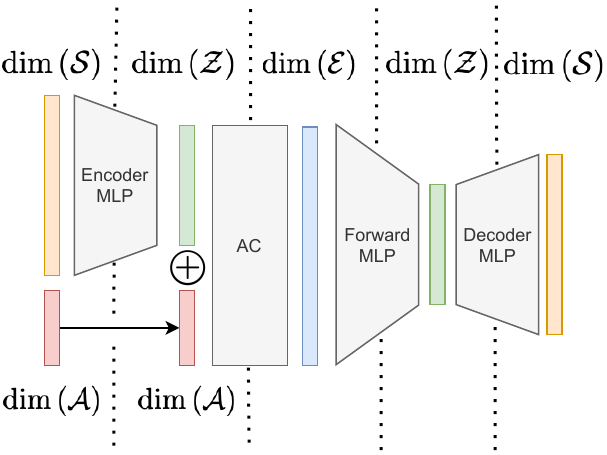}
\caption{\label{fig: arch} \small Latent model architecture (does not depict reward MLP). }
\end{figure} 

To further stabilize training we used a momentum encoder. Specifically we updated a target encoder as a slowly moving average of the weights from the learned encoder as
\begin{align*}
    \bar{\theta} \gets (1-\tau)\bar{\theta} + \tau \theta
\end{align*}
Hyperparameters for our variational model are summarized in \Cref{table: vae params}. The latent architecture is visualized in \Cref{fig: arch}. 

\begin{table}[h!]
\vskip 0.15in
\begin{center}
\begin{small}
\begin{sc}
\begin{tabular}{|cc|cc|}
\toprule
\midrule
\bf Parameter     & \bf Value   & \bf Parameter   & \bf Value \\ \hline
$\dim\brk*{\Z}$ & $32$   & Learning Rate  & $10^{-3}$   \\ \hline
$\dim\brk*{\mathcal{E}}$ & $32 + \dim\brk*{\A}$  & Batch Size  & $128$ (D4RL), $32$ (Highway-Env)   \\ \hline
Encoder MLP hidden & $256, 256$  & Target Update $\tau$  & 0.01   \\ \hline
Forward MLP hidden & $256, 256$  & Target Update Interval  & 1   \\ \hline
Decoder MLP hidden & $256, 256$   & Phase 1 Updates  & $100000$ (D4RL), $10000$ (Highway-Env)   \\ \hline
Reward MLP hidden  & $256, 256$   & Phase 2 Updates  & $50000$ (D4RL), $2000$ (Highway-Env)   \\ \hline
\bottomrule
\end{tabular}
\end{sc}
\end{small}
\end{center}
\caption{\label{table: vae params} Hyper parameters for variational model }
\vskip -0.1in
\end{table}%

\begin{figure}[t!]
\centering
\includegraphics[width=0.6\textwidth]{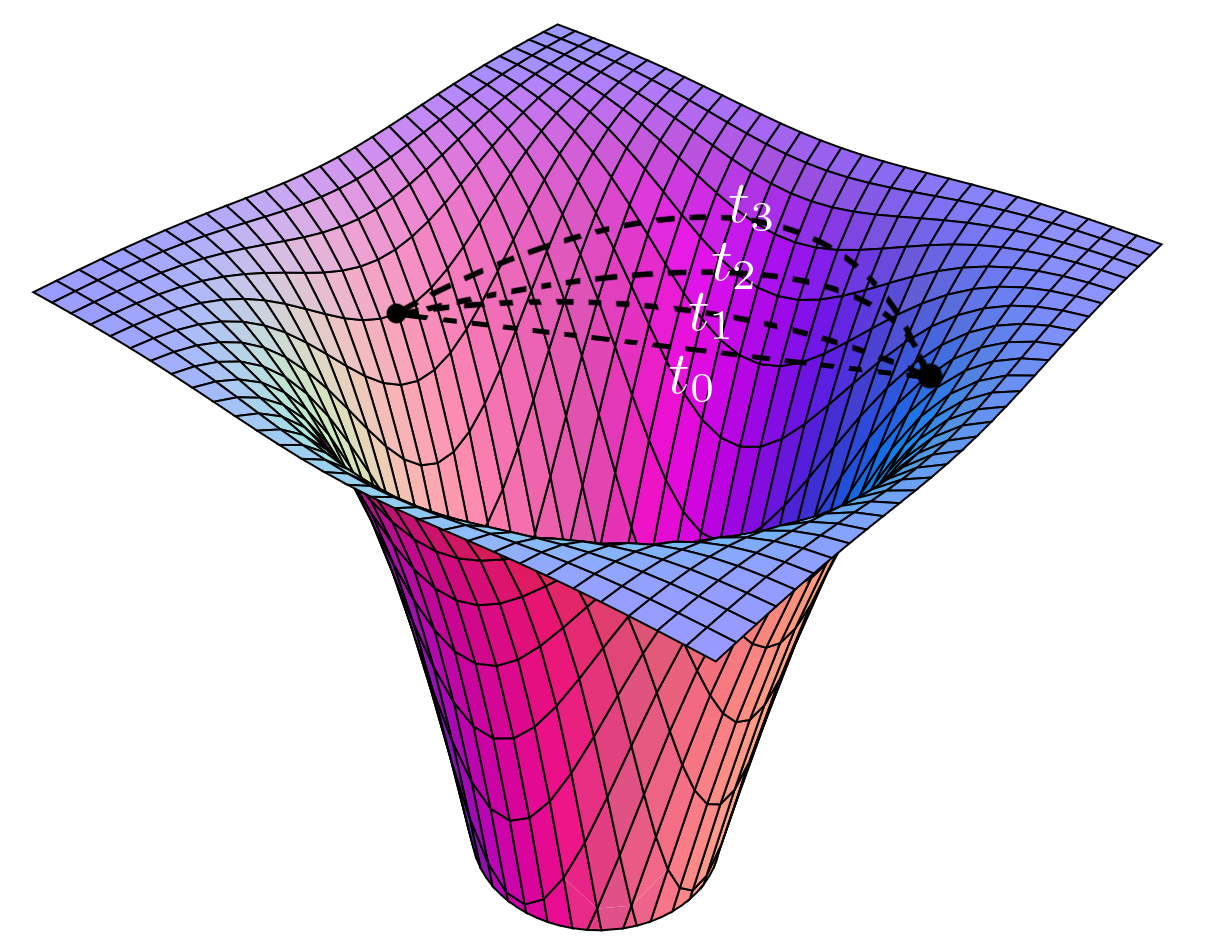}
\caption{\label{fig: curve convergence} \small Illustration of geodesic curve optimization in \Cref{alg: geodesic distance}. }
\end{figure} 

\subsection{Geodesic Distance Estimation}

In order to practically estimate the geodesic distance between two points $e_1, e_2 \in \mathcal{E}$ we defined a parametric curve in latent space and used gradient descent to minimize the curve's energy \footnote{Other methods for computing the geodesic distance include solving a system of ODEs \citep{arvanitidis2018latent}, using graph based geodesics \citep{chen2019fast}, or using neural networks \citep{chen2018metrics}.}. The resulting curve and pullback metric were then used to calculate the geodesic distance by a numerical estimate of the curve length. 

Pseudo code for Geodesic Distance Estimation is shown in \Cref{alg: geodesic distance}. Our curve was modeled as a cubic spline with $8$ coefficients. We used SGD (momentum $0.99$) to optimize the curve energy over $20$ gradient iterations with a grid of $10$ points and a learning rate of $10^{-3}$. An illustration of the convergence of such a curve is illustrated in \Cref{fig: curve convergence}

\begin{algorithm}[t!]
  \caption{Geodesic Distance Estimation}
  \label{alg: geodesic distance}
\begin{algorithmic}
  \STATE {\bfseries Input:} forward latent $F$, decoder $f_D$, learning rate $\lambda$, number of iterations $T$, grid size $n$, eval points $e_0, e_1$
  \STATE {\bfseries Initialize:} parametric curve $\gamma_\theta: \gamma_\theta(0) = e_0, \gamma_\theta(1) = e_1$ 
  \FOR{$t = 1$ {\bfseries to} $T$}
      \STATE $L_\mu(\theta) \gets \sum_{i=1}^n \mu_D(\mu_F(\gamma_\theta(\frac{i}{n}))) - \mu_D(\mu_F(\gamma_\theta(\frac{i-1}{n})))$
      \STATE $L_\sigma(\theta) \gets \sum_{i=1}^n \sigma_D(\mu_F(\gamma_\theta(\frac{i}{n}))) - \sigma_D(\mu_F(\gamma_\theta(\frac{i-1}{n})))$
      \STATE $L(\theta) \gets L_\mu(\theta) + L_\sigma(\theta)$
      \STATE $\theta \gets \theta - \lambda \nabla_\theta L(\theta)$
  \ENDFOR
  \STATE $G_{D \circ F} = J_{\mu_F}^T \overline{G}_DJ_{\mu_F}
        +
        J_{\sigma_F}^T \diag{\overline{G}_D}J_{\sigma_F}$
  \STATE $\forall i, \Delta_i \gets \gamma_\theta(\frac{i}{n}) - \gamma_\theta(\frac{i-1}{n})$
  \STATE $d(e_0, e_1) \gets \sum_{i=1}^{n} \brk*{\frac{\partial \gamma_\theta}{\partial t}\big\rvert_{\frac{i}{n}}}^T G_{D \circ F}(\gamma_\theta(\frac{i}{n})) \brk*{\frac{\partial \gamma_\theta}{\partial t}\big\rvert_{\frac{i}{n}}} \Delta_i$
  \STATE {\bfseries Return:} $d(e_0, e_1)$
\end{algorithmic}
\end{algorithm}

\subsection{RL algorithm}

\paragraph{D4RL.} Our learning algorithm is based on the Soft Learning framework proposed in Algorithm~2 of \citet{yu2020mopo}. Pseudo code is shown in \Cref{alg: gelato full}. Specifically we used two replay buffers $\mathcal{R}_{\text{model}}, \mathcal{R}_{\text{data}}$, where $\abs{R_{\text{model}}} = 50000$ and $\mathcal{R}_{\text{data}}$ contained the full offline dataset. In every epoch an initial state $s_0$ was sampled from the offline dataset and embedded using our latent model to generate $z_0 \in \Z$. During rollouts of $\pi$, embeddings $E_{z,a} \in \mathcal{E}$ were then generated from $z$ and used to (1) sample next latent state $z'$, (2) sample estimated rewards $r$, and (3) compute distances to $K=20$ nearest neighbors in embedded the dataset. 

\paragraph{Highway-Env.} We used PPO \citep{schulman2017proximal} implemented in RLlib \citep{liang2018rllib} trained over the variational forward model. All environments (except Racetrack) were trained with horizon $h=10$. The Racetrack environment was trained with horizon $h=50$.

We used \Cref{alg: geodesic distance} to compute the geodesic distances, and FAISS \citep{johnson2019billion} for efficient nearest neighbor computation on GPUs. To stabilize learning, we normalized the penalty $\frac{1}{K}\sum_{k=1}^K d_k$ according to the maximum penalty, ensuring penalty lies in $[0,1]$ (recall that the latent reward predictor was trained over normalized rewards in $[-1, 1]$). For the non-skewed version of GELATO, we used $\lambda = 1$ as our reward penalty coefficient and $\lambda=2$ for the skewed versions. We used rollout horizon of $h=5$, and did not notice significant performance improvement for different values of $h$.

\begin{algorithm}[t!]
  \caption{GELATO with Soft Learning}
  \label{alg: gelato full}
\begin{algorithmic}
  \STATE {\bfseries Input:} Reward penalty coefficient $\lambda$, rollout horizon $h$, rollout batch size $b$, training epochs $T$, number of neighbors $K$.
  \STATE Train variational latent forward model on dataset $\D_n$ by maximizing ELBO (\Cref{eq: elbo})
  \STATE Construct embedded dataset $\D_{\text{embd}} = \brk[c]*{E_i}_{i=1}^n$ using latent model to initialize KNN.
  \STATE Initialize policy $\pi$ and empty replay buffer $\mathcal{R}_{\text{model}} \gets \emptyset$.
  \FOR{epoch $ = 1$ {\bfseries to} $T$}
      \FOR{$i = 1$ {\bfseries to} $b$ (in parallel)}
        \STATE Sample state $s_1$ from $\D_n$ for the initialization of the rollout and embed using latent model to produce $z_1$.
        \FOR{$j = 1$ {\bfseries to} $h$}
            \STATE Sample an action $a_j \sim \pi(\cdot | z_j)$.
            \STATE Embed $(z_j, a_j) \to E_{z_j, a_j}$ using latent model
            \STATE Sample $z_{j+1}$ from latent forward model $F(E_{z_j, a_j})$.
            \STATE Sample $r_j$ from latent reward model $R(E_{z_j, a_j})$.
            \STATE Use \Cref{alg: geodesic distance} to compute $K$ nearest neighbors $\brk[c]*{\text{NN}^{(k)}_{z_j, a_j}}_{k=1}^K$ and their distances $\brk[c]*{d_k}_{k=1}^N$ to $E_{z_j, a_j}$.
            \STATE Compute $\tilde{r}_j = r_j - \lambda \brk*{\frac{1}{K}\sum_{k=1}^K d_k}$
            \STATE Add sample $(z_j , a_j , \tilde{r}_j , z_{j+1})$ to $\mathcal{R}_{\text{model}}$.
        \ENDFOR
      \ENDFOR
      \STATE Drawing samples from $\mathcal{R}_{\text{data}} \cup \mathcal{R}_{\text{model}}$, use SAC to update $\pi$.
  \ENDFOR
\end{algorithmic}
\end{algorithm}

\section{Visualization in Continuous Control Benchmarks}

\begin{figure*}[t!]
\captionsetup[subfigure]{labelformat=empty}
\centering
\begin{subfigure}{.4\textwidth}
    \centering
    \includegraphics[width=\textwidth]{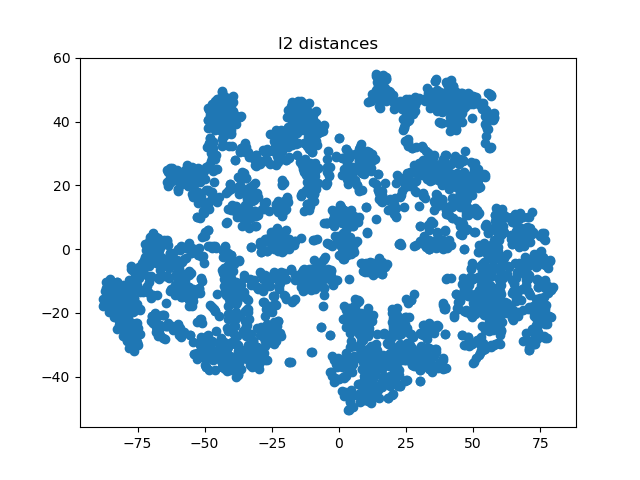}
\end{subfigure}
\begin{subfigure}{.4\textwidth}
    \centering
    \includegraphics[width=\textwidth]{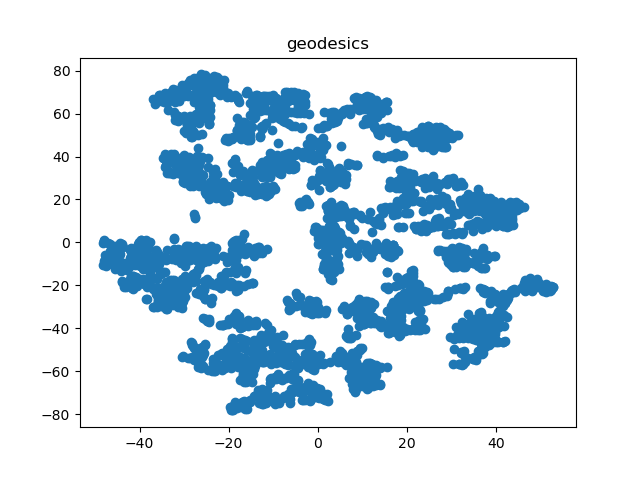}
\end{subfigure}
\caption{\label{fig: d4rl latent visualization} \small Plots show t-SNE embeddings generated using the HalfCheetah (medium) dataset. Left plot depicts embeddings using Euclidean distances. Right plot depicts embeddings using geodesics distances which induce curves of minimum energy in ambient space. Plots show first 3 episodes in the data.  }
\vskip -0.1in
\end{figure*} 

In \Cref{sec: metric visualization} we showed visualizations of our metric and compared its geodesics to $\ell_2$ distances in latent space. In \Cref{fig: d4rl latent visualization} we show similar visualization for the medium-agent Halfcheetah dataset in D4RL. Particually, it is evident that the latent state in this environment is similar under our proposed metric and the standard Euclidean distance, suggesting the problems' natural manifold is in fact flat \citep{chen2020learning}. These results give evidence to performance results of GELATO in \Cref{table: results}. Indeed, GELATO with $\ell_2$ penalty achieved similar performance to other methods in the D4RL benchmarks. Nevertheless, as was shown in \Cref{fig: highway env results}, the Euclidean distance in latent space failed catastrophically in the autonomous driving environments, due to discontinuities in latent space (as depicted in \Cref{fig: highway embeddings}.

\section{Missing Proofs}

\subsection{Proof of \Cref{prop: metric}}

For any curve $\gamma$, we have that
\begin{align*}
    \int_0^1 \norm{\frac{\partial f(\gamma(t))}{\partial t}} dt
    &=
    \int_0^1 \norm{\frac{\partial f(\gamma(t))}{\partial \gamma(t)}^T\frac{\partial \gamma(t)}{\partial t}} dt \\
    &=
    \int_0^1 \norm{J_f(\gamma(t))^T\frac{\partial \gamma(t)}{\partial t}} dt \\
    &=
    \int_0^1 \sqrt{\brk[a]*{\frac{\partial \gamma(t)}{\partial t}, J_f^T(\gamma(t)) J_f(\gamma(t)) \frac{\partial \gamma(t)}{\partial t}}} dt \\
    &=
    \int_0^1 \sqrt{\brk[a]*{\frac{\partial \gamma(t)}{\partial t}, G_f(\gamma(t))  \frac{\partial \gamma(t)}{\partial t}}} dt.
\end{align*}
This completes the proof.

\subsection{Proof of Theorems~1 and~2}

We begin by restating the theorems.

\metricdecoderprop*
\metriccompprop*

Notice that \Cref{prop: decoder metric} is a special case of \Cref{prop: metric comp} with $F$ being the trivial identity function. Additionally, a complete proof of \Cref{prop: decoder metric} can be found in \citet{arvanitidis2018latent}. We turn to prove \Cref{prop: metric comp}.

We begin by proving the following auxilary lemma.
\begin{lemma}
\label{lemma: aux hessian}
    Let $\epsilon \sim \mathcal{N}(0, I_{K})$, $f: \R^d \mapsto \R^K$, $A \in \R^{K \times K}$. Denote $S_i = \diag{\frac{\partial f^1}{\partial z_i}, \frac{\partial f^2}{\partial z_i}, \hdots, \frac{\partial f^K}{\partial z_i}}$ for $1 \leq i \leq d$ and
    \begin{align*}
        B
        = 
    \begin{bmatrix}
        S_1 \epsilon, 
        S_2 \epsilon,
        \hdots,
        S_d \epsilon
    \end{bmatrix}_{K \times d}.
    \end{align*}
    Then 
    $\E{B^TAB} = J^T_f \diag{A} J_f$.
\end{lemma}
\begin{proof}
    We have that
    \begin{align*}
    \E{B^TAB}
    &=
    \E{
    \begin{bmatrix}
        \epsilon^T S_1^T \\
        \epsilon^T S_2^T \\
        \vdots \\
        \epsilon^T S_d^T
    \end{bmatrix}
    A
    \begin{bmatrix}
        S_1 \epsilon, 
        S_2 \epsilon,
        \hdots,
        S_d \epsilon
    \end{bmatrix}} \\
    &=
    \begin{bmatrix}
        \E{\epsilon^TS_1^TAS_1\epsilon},
        \E{\epsilon^TS_1^TAS_2\epsilon},
        \hdots,
        \E{\epsilon^TS_1^TAS_d\epsilon}, \\
        \E{\epsilon^TS_2^TAS_1\epsilon},
        \E{\epsilon^TS_2^TAS_2\epsilon},
        \hdots,
        \E{\epsilon^TS_2^TAS_d\epsilon},\\
        \ddots \\
        \E{\epsilon^TS_d^TAS_1\epsilon},
        \E{\epsilon^TS_d^TAS_2\epsilon}
        \hdots,
        \E{\epsilon^TS_d^TAS_d\epsilon}
    \end{bmatrix}.
    \end{align*}
    
    Finally notice that for any matrix $M$
    \begin{align*}
        \E{\epsilon^TM\epsilon}
        =
        \sum_{i=1}^d \sum_{j=1}^d M_{ij}\E{\epsilon_i \epsilon_j}
        =
        \sum_{i=1}^d M_{ii}
        =
        \trace{M}.
    \end{align*}
    Then,
    \begin{align*}
        \E{B^TAB}
        &=
        \begin{bmatrix}
            \trace{S_1^TAS_1},
            \trace{S_1^TAS_2},
            \hdots,
            \trace{S_1^TAS_d} \\
            \trace{S_2^TAS_1},
            \trace{S_2^TAS_2},
            \hdots,
            \trace{S_2^TAS_d} \\
            \ddots \\
            \trace{S_d^TAS_1},
            \trace{S_d^TAS_2},
            \hdots,
            \trace{S_d^TAS_d} \\
        \end{bmatrix}.
    \end{align*}
\end{proof}
Next, note that
\begin{align*}
    \trace{S_iAS_j}
    =
    \sum_{k=1}^K \frac{\partial f^k}{\partial z_i}\frac{\partial f^k}{\partial z_j}A_{kk}.
\end{align*}
Therefore,
\begin{align*}
    \E{B^TAB}
    =
    J_f^T\diag{A}J_f.
\end{align*}

We are now ready to prove \Cref{prop: metric comp}.
\begin{proof}[Proof of \Cref{prop: metric comp}]

We can write $z' = \mu_F(z) + \sigma_F(z) \odot \epsilon_F$ and $s' = \mu_D(z') + \sigma_D(z') \odot \epsilon_D$ where $\epsilon_F \sim \mathcal{N}(0, I_d), \epsilon_D \sim \mathcal{N}(0, I_K)$, ${\mu_F:\R^d \mapsto \R^\ell,} {\mu_D : \R^\ell \mapsto \R^K}$ and $\sigma_F: \R^d \mapsto \R^\ell, \sigma_D: \R^\ell \mapsto \R^K$. 

Applying the chain rule we get
\begin{align*}
    J_{D \circ F} = \frac{\partial (D \circ F)}{\partial z} = J_{\mu_D}J_{\mu_F} + J_{\mu_D}B_{\epsilon_F} + B_{\epsilon_D}J_{\mu_F} + B_{\epsilon_D}B_{\epsilon_F}
\end{align*}
where
\begin{align*}
    B_{\epsilon_F} 
    &= 
    \begin{pmatrix}
        S_{F,1} \epsilon_F, 
        S_{F,2} \epsilon_F,
        \hdots,
        S_{F,d} \epsilon_F
    \end{pmatrix}_{d \times d}, \\
    S_{F,i} 
    &= 
    \diag{\frac{\partial \sigma_F^1}{\partial z_i}, \frac{\partial \sigma_F^2}{\partial z_i}, \hdots, \frac{\partial \sigma_F^d}{\partial z_i}}_{d \times d},
    \quad \text{       and}\\
    B_{\epsilon_D} 
    &= 
    \begin{pmatrix}
        S_{D,1} \epsilon_D, 
        S_{D,2} \epsilon_D,
        \hdots,
        S_{D,d} \epsilon_D
    \end{pmatrix}_{K \times d}, \\
    S_{D,i} 
    &= 
    \diag{\frac{\partial \sigma_D^1}{\partial z'_i}, \frac{\partial \sigma_D^2}{\partial z'_i}, \hdots, \frac{\partial \sigma_D^d}{\partial z'_i}}_{K \times K}.
\end{align*}
This yields
\begin{align*}
    \E{J_{F \circ D}^T J_{F \circ D}}
    &=
    \E{\brk[r]*{J_{\mu_D}J_{\mu_F} + J_{\mu_D}B_{\epsilon_F} + B_{\epsilon_D}J_{\mu_F} + B_{\epsilon_D}B_{\epsilon_F}}^T
    \brk[r]*{J_{\mu_D}J_{\mu_F} + J_{\mu_D}B_{\epsilon_F} + B_{\epsilon_D}J_{\mu_F} + B_{\epsilon_D}B_{\epsilon_F}}} \\
    &=
    J_{\mu_F}^T J_{\mu_D}^T J_{\mu_D} J_{\mu_F}
    +
    \E{B_{\epsilon_F}^T J_{\mu_D}^T J_{\mu_D}B_{\epsilon_F}}
    +
    \E{J_{\mu_F}^T B_{\epsilon_D}^T B_{\epsilon_D}J_{\mu_F}}
    +
    \E{B_{\epsilon_F}^T B_{\epsilon_D}^T B_{\epsilon_D}B_{\epsilon_F}} \\
    &=
    J_{\mu_F}^T
    \brk[r]*{J_{\mu_D}^T J_{\mu_D} + \E{B_{\epsilon_D}^T B_{\epsilon_D}}}
    J_{\mu_F}
    +
    \E{B_{\epsilon_F}^T
    \brk[r]*{J_{\mu_D}^T J_{\mu_D} + B_{\epsilon_D}^T B_{\epsilon_D}}
    B_{\epsilon_F}}
\end{align*}
where in the second equality we used the fact that $\epsilon_D, \epsilon_F$ are independent and $\E{B_\epsilon} = 0$. 
By \Cref{lemma: aux hessian} we have
\begin{align*}
    \E{B_{\epsilon_D}^T B_{\epsilon_D}}
    = 
    J_{\sigma_D}^TJ_{\sigma_D}.
\end{align*}
Similarly,
\begin{align*}
    \E{B_{\epsilon_F}^T B_{\epsilon_D}^T B_{\epsilon_D}B_{\epsilon_F}}
    =
    \E{\E{B_{\epsilon_F}^T B_{\epsilon_D}^T B_{\epsilon_D}B_{\epsilon_F} | \epsilon_F}}
    =
    \E{B_{\epsilon_F}^T J_{\sigma_D}^TJ_{\sigma_D} B_{\epsilon_F}}
    = 
    J_{\sigma_F}^T\diag{J_{\sigma_F}^TJ_{\sigma_F}}J_{\sigma_F}
\end{align*}
Finally,
\begin{align*}
    \E{B_{\epsilon_F}^T
    J_{\mu_D}^T J_{\mu_D} 
    B_{\epsilon_F}}
    =
    J_{\sigma_F}^T\diag{J_{\mu_D}^T J_{\mu_D}}J_{\sigma_F}.
\end{align*}
The proof is complete by taking expectation with respect to uniformly distributed set of decoders $\brk[c]*{D_i}_{i=1}^M$.
\end{proof}

\end{document}